\documentclass[11pt]{article}

\usepackage[margin=1in]{geometry}
\usepackage{amsmath,amssymb,amsthm,mathtools}
\usepackage{array}
\usepackage{arydshln}
\usepackage{booktabs}
\usepackage{caption}
\usepackage{enumitem}
\usepackage{float}
\usepackage{algorithm}
\usepackage{algpseudocode}
\usepackage[authoryear,round]{natbib}
\usepackage[colorlinks=true,linkcolor=blue,citecolor=blue,urlcolor=blue]{hyperref}
\usepackage{aliascnt}
\usepackage[nameinlink,noabbrev]{cleveref}
\usepackage{microtype}
\usepackage[most]{tcolorbox}

\newtheorem{theorem}{Theorem}[section]

\newaliascnt{lemma}{theorem}
\newtheorem{lemma}[lemma]{Lemma}
\aliascntresetthe{lemma}

\newaliascnt{proposition}{theorem}
\newtheorem{proposition}[proposition]{Proposition}
\aliascntresetthe{proposition}

\newaliascnt{corollary}{theorem}
\newtheorem{corollary}[corollary]{Corollary}
\aliascntresetthe{corollary}

\theoremstyle{definition}
\newaliascnt{definition}{theorem}
\newtheorem{definition}[definition]{Definition}
\aliascntresetthe{definition}

\newaliascnt{remark}{theorem}

\aliascntresetthe{remark}

\newaliascnt{example}{theorem}

\aliascntresetthe{example}

\crefname{theorem}{Theorem}{Theorems}
\crefname{lemma}{Lemma}{Lemmas}
\crefname{proposition}{Proposition}{Propositions}
\crefname{corollary}{Corollary}{Corollaries}
\crefname{definition}{Definition}{Definitions}
\crefname{remark}{Remark}{Remarks}
\crefname{example}{Example}{Examples}
\crefname{section}{Section}{Sections}
\crefname{figure}{Figure}{Figures}
\crefname{algorithm}{Algorithm}{Algorithms}

\newcommand{\Ftwo}{\mathbb F_2}
\newcommand{\cX}{\mathcal{X}}
\newcommand{\cY}{\mathcal{Y}}
\newcommand{\cE}{\mathcal{E}}
\newcommand{\cA}{\mathcal{A}}
\newcommand{\cF}{\mathcal{F}}
\newcommand{\cH}{\mathcal{H}}
\newcommand{\cG}{\mathcal{G}}
\newcommand{\cT}{\mathcal{T}}
\newcommand{\cB}{\mathcal{B}}
\newcommand{\cU}{\mathcal{U}}
\newcommand{\Roll}{\operatorname{Roll}}
\newcommand{\ParDim}{\operatorname{ParDim}}
\newcommand{\DSdim}{\operatorname{DSdim}}
\newcommand{\VCdim}{\operatorname{VCdim}}
\newcommand{\Ldim}{\operatorname{Ldim}}
\newcommand{\Monline}{M_{\mathrm{online}}}
\newcommand{\dens}{\operatorname{dens}}
\newcommand{\codim}{\operatorname{codim}}
\newcommand{\supp}{\operatorname{supp}}
\newcommand{\Blk}{\operatorname{Blk}}
\newcommand{\Ch}{\operatorname{Ch}}
\newcommand{\PAC}{\operatorname{PAC}}
\newcommand{\halt}{\mathtt{halt}}
\newcommand{\eps}{\epsilon}
\algrenewcommand\algorithmicrequire{\textbf{Input:}}
\algrenewcommand\algorithmicensure{\textbf{Output:}}
\algtext*{EndWhile}

\newtcolorbox{questionbox}[1][]{
  enhanced,
  colback=white,
  colframe=black!45,
  boxrule=0.6pt,
  arc=2pt,
  left=6pt,
  right=6pt,
  top=6pt,
  bottom=6pt,
  fonttitle=\bfseries,
  coltitle=black,
  title=#1
}

\newtcolorbox{examplebox}[1][]{
  enhanced,
  breakable,
  colback=white,
  colframe=black!35,
  boxrule=0.5pt,
  arc=2pt,
  left=6pt,
  right=6pt,
  top=5pt,
  bottom=5pt,
  fonttitle=\bfseries,
  coltitle=black,
  title=#1
}

\title{The Optimal Sample Complexity of Learning Autoregressive Chain-of-Thought}
\author{Zhiyuan Li\\Toyota Technological Institute at Chicago\\\texttt{zhiyuanli@ttic.edu}}
\date{}

\begin{document}
\maketitle

\begin{abstract}
We prove that, in the realizable PAC setting, the sample complexity of
exact-trace learning for full autoregressive Chain-of-Thought traces is upper
bounded by the standard multiclass rate of the local next-token class, where
this rate is governed by the Daniely--Shalev-Shwartz dimension.  Under
exact-trace loss, one wrong action makes the whole trace incorrect;
nevertheless, for every stopping rule \(\halt\) and every pointwise
\(\halt\)-halting local class \(\cH\),
\[
        n_{\PAC}^{\varepsilon,\delta}(\Roll_\halt(\cH))
        =
        O\!\left(
        \frac{\DSdim(\cH)+\log(1/\delta)}{\varepsilon}
        \right),
\]
with no dependence on rollout length.  The dependence on \(\DSdim(\cH)\) is
worst-case optimal, since one-step stopping recovers ordinary multiclass
learning of \(\cH\).

The proof introduces parity dimension, a rollout-stable refinement of DS
dimension based on even pseudo-cubes.  It controls one-inclusion density via a
low-coordinate spanning theorem on finite restrictions and, unlike DS
dimension itself, does not increase under autoregressive rollout.  We also
show why this detour is necessary: DS dimension can increase under rollout.
\end{abstract}

\section{Introduction}
\label{sec:introduction}

Chain-of-Thought supervision exposes more than a final answer: it gives the
intermediate tokens, actions, or reasoning steps along a supervised trace
\citep{Joshi2025,HannekeMehalelMoran2026}.
In an autoregressive model, these coordinates are generated by a single shared
local rule.  The rule is queried at the current state, its action is appended
to the transcript, and the process repeats until a stopping condition is met.
Under exact-trace loss, the learner is correct only if the entire generated
trace is correct.  One wrong action makes the whole trace wrong.

At first sight, this all-or-nothing loss over a variable-length sequence should
increase sample complexity.  The label space consists of complete action
strings, and exact correctness requires every coordinate of the generated trace
to be right.  The coordinates, however, are not chosen independently: they are
produced by repeatedly querying the same local next-action rule.  The central
question is whether this shared autoregressive structure completely removes
the apparent statistical cost of predicting an entire trace.

\paragraph{Problem formulation.}
We write \(\Roll_\halt(\cH)\) for the complete-trace class obtained by rolling
out a local next-action class \(\cH\) under a deterministic stopping rule
\(\halt\).  Formally, the state space has aligned form
\(\cX=\cE\times\cA^*\), the local class is \(\cH\subseteq \cA^\cX\), and
\(\halt\) decides when to stop from the initial state and the emitted suffix.
The environment coordinate \(\cE\) stores exogenous context that can affect the
future rollout but is not necessarily emitted as an action string.  It may
include the prompt, task state, tool observations in interactive language-model
agents such as \citet{Yao2023ReAct}, or visual context represented by
multimodal encoders such as \citet{Radford2021CLIP}.  The rollout model is
formalized in \Cref{def:aligned-state,def:pointwise-halting,def:rollout-class}.
Exact-trace learning is ordinary multiclass learning of
\(\Roll_\halt(\cH)\subseteq(\cA^*)^\cX\) by \Cref{prop:exact-trace-is-rollout},
whose labels are complete action strings.

\paragraph{Existing routes and the optimal-rate question.}
For ordinary multiclass PAC learning, the sharp realizable benchmark is the
Daniely--Shalev-Shwartz dimension \(\DSdim\), recalled in
\Cref{def:ds-dimension}.  Since exact-trace learning is multiclass learning of
\(\Roll_\halt(\cH)\), the general multiclass theory already characterizes the
rollout problem in terms of \(\DSdim(\Roll_\halt(\cH))\).  Thus the question is
not whether the rollout class is PAC learnable, but how its complexity compares
with that of the local next-action class \(\cH\).

Indeed, if the stopping rule stops after one action, then
\(\Roll_\halt(\cH)\) is just \(\cH\), with labels written as length-one
strings.  Therefore any theorem stated in terms of the local class must in
general pay at least the ordinary multiclass complexity of \(\cH\), which was
recently proved by \citet{Pabbaraju2026} to be the local-class rate
\[
        O\!\left(
        \frac{\DSdim(\cH)+\log(1/\delta)}{\varepsilon}
        \right).
\]
In the binary case this benchmark is the usual VC-dimension rate.

Existing work gives several routes toward learning autoregressive
Chain-of-Thought traces, summarized in \Cref{tab:cot-rates}, but none gives
this sharp local PAC rate.  The proper trace-consistency route of
\citet{Joshi2025} gives a PAC bound with logarithmic dependence on the rollout
length \(T\), and hence is not the sharp local PAC rate.  The compression route
of \citet{HannekeMehalelMoran2026} removes explicit \(T\)-dependence, but pays
a larger complexity factor involving \(\VCdim(\cH)\VCdim^\star(\cH)\), up to
polylogarithmic factors.  The online route is also length-independent:
complete-rollout mistakes can be charged to local next-action mistakes,
yielding bounds in terms of \(\Ldim(\cH)\); see
\citet{DoronAradMehalelMossel2026} and
\citet{BalcanBlumFragkiaLiSharma2026}.  For completeness,
\Cref{app:online-comparison} restates this online reduction in the notation of
the present setup.  However, \(\Ldim\) is an online complexity parameter and
can be much larger than the realizable PAC parameter \(\DSdim\), indeed
infinite even when \(\DSdim\) is finite, as for threshold classes.  Thus the
online route does not settle the PAC question.  Since exact-trace prediction is
stronger than final-token prediction, a positive answer to our question would
in particular resolve the optimal-rate question raised by
\citet{HannekeMehalelMoran2026} for the corresponding CoT-supervision setting.

\begin{questionbox}[Question]
In the realizable PAC setting, is learning full autoregressive
Chain-of-Thought traces statistically as easy as learning the local next-token
rule that generates them?
\end{questionbox}

\begin{table}[t]
\centering
\setlength{\tabcolsep}{4pt}
\renewcommand{\arraystretch}{1.15}
\setlength{\dashlinedash}{1.2pt}
\setlength{\dashlinegap}{2.2pt}
\begin{tabular}{@{}
  >{\raggedright\arraybackslash}m{0.40\textwidth}
  >{\centering\arraybackslash}m{0.54\textwidth}
  @{}}
\toprule
Learning rule & PAC rate for exact-trace autoregressive rollout \\
\midrule
Proper trace-consistency / ConsCoT\\[-1pt]
{\scriptsize \citet{Joshi2025}}
&
\(O\!\bigl((\VCdim(\cH)\log T\log(1/\varepsilon)+\log(1/\delta))/\varepsilon\bigr)\)
\tabularnewline \hdashline
Compression-based CoT learner\\[-1pt]
{\scriptsize \citet{HannekeMehalelMoran2026}}
&
\(\widetilde O\!\bigl((\VCdim(\cH)\VCdim^\star(\cH)\log(1/\varepsilon)+\log(1/\delta))/\varepsilon\bigr)\)
\tabularnewline \hdashline
Online reduction \(+\) online-to-batch\\[-1pt]
{\scriptsize \citet{DoronAradMehalelMossel2026};
\citet{BalcanBlumFragkiaLiSharma2026};
\hyperref[app:online-comparison]{Appendix~\ref*{app:online-comparison}}}
&
\(O\!\bigl((\Ldim(\cH)\log(1/\varepsilon)+\log(1/\delta))/\varepsilon\bigr)\)
\tabularnewline \hdashline
Standard one-inclusion \(+\) transductive-to-batch conversion\\[-1pt]
{\scriptsize (\textbf{This paper}; \Cref{thm:pac-bound})}
&
\(O\!\bigl((\DSdim(\cH)+\log(1/\delta))/\varepsilon\bigr)\)
\tabularnewline \hdashline
Any learning rule / lower bound\\[-1pt]
{\scriptsize one-step stopping; \Cref{cor:worst-case-optimality} \(+\) DS lower bound
\citep{DanielyShalev2014,Pabbaraju2026}}
&
\(\Omega\!\bigl((\DSdim(\cH)+\log(1/\delta))/\varepsilon\bigr)\)
\\
\bottomrule
\end{tabular}
{\captionsetup{font=scriptsize}
\caption{Comparison routes for exact-trace autoregressive Chain-of-Thought
learning.  Here \(T\) denotes a rollout-length bound when such a bound is
assumed, and \(\VCdim^\star\) denotes dual VC dimension.  The
\(\widetilde O\) in the second row hides polylogarithmic factors in
\(\VCdim(\cH)\) and \(\VCdim^\star(\cH)\).}
\label{tab:cot-rates}}
\end{table}

\paragraph{Main result.}
We answer this question affirmatively.  In fact, the proof gives a slightly
sharper bound in terms of parity dimension, a parameter introduced below.

\begin{theorem}[Autoregressive exact-trace PAC bound]
\label{thm:pac-bound}
Let \(\cE\) be an environment space, \(\cA\) an action space,
\(\cX=\cE\times\cA^*\), and
\(\halt:\cX\times\cA^*\to\{0,1\}\) a stopping rule.  If
\(\cH\subseteq \cA^\cX\) is pointwise \(\halt\)-halting, then
\[
        n_{\PAC}^{\varepsilon,\delta}(\Roll_\halt(\cH))
        =
        O\!\left(
        \frac{\ParDim(\cH)+\log(1/\delta)}{\varepsilon}
        \right)
        \le
        O\!\left(
        \frac{\DSdim(\cH)+\log(1/\delta)}{\varepsilon}
        \right).
\]
Moreover, the dependence on \(\DSdim(\cH)\) is worst-case optimal: one-step
stopping recovers ordinary realizable multiclass learning of \(\cH\), and hence
inherits the usual DS-dimension lower bound.
\end{theorem}

\paragraph{The key invariant.}
The proof goes through parity dimension \(\ParDim\), defined formally in
\Cref{def:parity-dimension}.  This is a DS-style dimension in which ordinary
pseudo-cube witnesses are replaced by \(\Ftwo\)-even pseudo-cube witnesses.
Where a DS pseudo-cube asks, point by point, for a disagreeing neighbor in each
coordinate, an even pseudo-cube asks for a nonzero \(\Ftwo\)-weighted
collection whose marginal on every one-coordinate deletion is zero.  This
linear cancellation condition turns a witness into a tensor; the proof overview
below explains how this tensor is used.  A bare existence witness for DS
dimension has no comparable operation.

Parity dimension has the two properties needed for the proof.  First, it sits between
one-inclusion density and DS dimension by
\Cref{thm:parity-density,cor:pardim-ds}:
\[
        \mu_\cF(n)\le \ParDim(\cF)\le\DSdim(\cF)
        \qquad \forall n\ge1.
\]
Second, it does not increase under autoregressive rollout by
\Cref{thm:parity-rollout-nonincrease}:
\[
        \ParDim(\Roll_\halt(\cH))\le \ParDim(\cH).
\]
Combining these inequalities gives
\[
        \mu_{\Roll_\halt(\cH)}(n)
        \le
        \ParDim(\Roll_\halt(\cH))
        \le
        \ParDim(\cH)
        \le
        \DSdim(\cH).
\]
The one-inclusion and leave-one-out-to-PAC machinery then gives the PAC bound
above.

\paragraph{Proof overview.}
The proof has two independent parts.  The first is a finite spanning theorem,
proved in \Cref{sec:finite-parspan-density}.
On a finite restriction \(D\), absence of large even pseudo-cubes is equivalent
to vanishing of high-order marginal annihilators.  We prove that this vanishing
forces all functions on the finite class to be spanned by low-coordinate
functions (\Cref{lem:parspan-vanishing}).  A basis-counting argument then
bounds one-inclusion density (\Cref{lem:touch-counting-density}), yielding
\Cref{thm:parity-density}.  This
plays the role that low-degree polynomial spanning plays in the optimal
multiclass Sauer and density theorems of \citet{HannekeMengMoranShaeiri2026}
and \citet{Pabbaraju2026}, but the spanning objects here are \(\Ftwo\)-valued
low-coordinate functions.  Thus \Cref{thm:parity-density} refines Pabbaraju's
DS-density bound: the DS bound follows from \(\ParDim\le\DSdim\), and the
refinement is strict on the rollout counterexample by
\Cref{cor:pardim-ds-separation}.

The second part is specific to rollout and is proved in
\Cref{sec:parity-rollout-nonincrease}.  A parity certificate for a finite
restriction of \(\Roll_\halt(\cH)\) is a tensor on complete traces.  For each
coordinate, the possible traces form a prefix partition tree.  A
partition-tree peeling theorem (\Cref{thm:product-peeling}) moves the tensor
down these trees until each coordinate is a choice among children of one
prefix.  Such children are one-step next actions, so
\Cref{thm:trace-to-base} turns the resulting tensor into a parity certificate
for \(\cH\), proving \Cref{thm:parity-rollout-nonincrease}.

\paragraph{Why DS dimension is not enough.}
A tempting proof strategy would be to show that DS dimension itself is
preserved by rollout, \(\DSdim(\Roll_\halt(\cH))\le\DSdim(\cH)\), and then
apply the standard multiclass density theorem.  This route fails.  We construct
a finite aligned autoregressive class with \(\DSdim(\cH)=2\) and
\(\DSdim(\Roll_\halt(\cH))\ge3\); the formal statement is
\Cref{thm:ds-rollout-increase}, and the construction and verification are in
\Cref{sec:ds-blowup}.  The obstruction is structural: a trace-level
pseudo-cube may use disagreements that occur at incompatible prefix depths, so
it need not descend to a pseudo-cube of fixed next-action states.

\begin{examplebox}[Example 1: Chain-of-Thought reasoning]
For text Chain-of-Thought, let \(\Sigma\) be a token alphabet.  Prompts and
other fixed context are included in the initial transcript.  Thus we take
\(\cA=\Sigma\), let \(\eps\) denote the empty string, set
\(\cX=\{\eps\}\times\Sigma^*\), and let
\(\cF\subseteq\Sigma^\cX\) be a next-token class.  A state is written
\((\eps,w)\), where \(w\in\Sigma^*\) is the current transcript.

\begin{theorem}[Chain-of-Thought]
\label{thm:intro-text-cot}
For every stopping rule \(\halt:\cX\times\Sigma^*\to\{0,1\}\), if \(\cF\) is
pointwise \(\halt\)-halting in the sense of \Cref{def:pointwise-halting}, then
\[
        n_{\PAC}^{\varepsilon,\delta}(\Roll_\halt(\cF))
        =
        O\!\left(\frac{\ParDim(\cF)+\log(1/\delta)}{\varepsilon}\right)
        \le
        O\!\left(\frac{\DSdim(\cF)+\log(1/\delta)}{\varepsilon}\right).
\]
Thus the rate is insensitive to the stopping convention, provided rollouts
terminate.  For the EOS variants, assume \(\mathrm{[EOS]}\in\Sigma\).  This
includes:
\begin{itemize}[leftmargin=2.2em,itemsep=2pt,topsep=2pt]
\item fixed total transcript length:
\(\halt_T^{\mathrm{tot}}((\eps,w),y)=1\Longleftrightarrow |wy|\ge T\);
\item fixed output length:
\(\halt_T^{\mathrm{out}}((\eps,w),y)=1\Longleftrightarrow |y|\ge T\);
\item EOS stopping:
\(\halt_{\mathrm{eos}}((\eps,w),y)=1\Longleftrightarrow y\) ends in
\(\mathrm{[EOS]}\);
\item capped EOS stopping:
\(\halt_{T,\mathrm{eos}}((\eps,w),y)=1\Longleftrightarrow |wy|\ge T\) or
\(y\) ends in \(\mathrm{[EOS]}\).
\end{itemize}
\end{theorem}
This is the specialization of \Cref{thm:pac-bound} to
\(\cX=\{\eps\}\times\Sigma^*\), \(\cA=\Sigma\), and \(\cH=\cF\).
\end{examplebox}

\begin{examplebox}[Example 2: Full-label multi-instance learning]
Let \(\cU\) be a local instance space, \(\cY\) a local label space, and
\(\cG\subseteq\cY^\cU\) a local multiclass class.  Its full-label lift
\(\cG^\star\) maps a finite list \(U=(u_1,\ldots,u_T)\) to the complete label
list
\[
        g^\star(U)=(g(u_1),\ldots,g(u_T)).
\]
Under all-or-nothing loss on the whole label list, the optimal realizable PAC
sample complexity is unchanged:
\[
        n_{\PAC}^{\varepsilon,\delta}(\cG^\star)
        =
        \Theta\!\left(
        \frac{\DSdim(\cG)+\log(1/\delta)}{\varepsilon}
        \right)
        \quad(\DSdim(\cG)\ge1).
\]
The formal statement, the distinction from bag-label MIL, and the random ERM
lower bound are given in \Cref{sec:full-label-mil-application}.
\end{examplebox}

\paragraph{Organization.}
\Cref{sec:model} formalizes autoregressive exact-trace learning and recalls the
multiclass PAC and density tools.  \Cref{sec:parspan} introduces parity
dimension and states the main structural theorems.  \Cref{sec:finite-parspan-density}
proves low-coordinate spanning on finite restrictions and derives the density bound.
\Cref{sec:parity-rollout-nonincrease} proves that parity dimension does not
increase under rollout by partition-tree peeling.
\Cref{sec:full-label-mil-application} gives the full-label multi-instance
learning application in its original sequence-lift form, and
\Cref{app:random-erm-mil-lower} proves the random ERM lower bound.
\Cref{app:dimension-comparisons} contains the elementary dimension comparisons,
\Cref{app:online-comparison} records the online mistake-bound comparison, and
\Cref{sec:ds-blowup} gives the finite
construction showing that DS dimension itself can increase under rollout.

\section{Autoregressive exact-trace learning and multiclass preliminaries}
\label{sec:model}

This section fixes the learning model and the multiclass notation used in the
rest of the paper.  We first define aligned autoregressive states, stopping
rules, rollout maps, and exact-trace loss.  The key point is that exact-trace
learning is an ordinary multiclass learning problem for the rollout class.  We
then recall the PAC, transductive leave-one-out, one-inclusion density, and
DS-dimension facts used later.

For a set \(\cA\), write \(\cA^*\) for the set of finite strings over \(\cA\),
write \(\eps\) for the empty string, and write \(wu\) for concatenation of
strings \(w,u\in\cA^*\).  For \(n\ge1\), write
\([n]=\{1,\ldots,n\}\).  For a function class \(\cH\subseteq\cY^\cX\) and a
domain subset \(D\subseteq\cX\), write
\[
        \cH|_D=\{h|_D:h\in\cH\}\subseteq\cY^D.
\]
If \(D'\subseteq D\) and \(u\in\cH|_D\), write \(u|_{D'}\) for the ordinary restriction of \(u\) to \(D'\).

\subsection{Autoregressive exact-trace model}

\begin{definition}[Aligned states and stopping predicates]
\label{def:aligned-state}
Let \(\cE\) be an environment space and \(\cA\) an action space.  The aligned state domain is \(\cX=\cE\times\cA^*\).  For \(x=(e,w)\in \cX\) and \(u\in\cA^*\), write \(x\cdot u=(e,wu)\).  A stopping predicate is a function
\[
        \halt:\cX\times\cA^*\to\{0,1\},
\]
where \(\halt(x,u)\) is evaluated at the initial aligned state \(x\) and the
emitted suffix \(u\).
\end{definition}

\begin{algorithm}[H]
\caption{\(\Roll_\halt(h)(x)\): deterministic autoregressive rollout}
\label{alg:autoregressive-rollout}
\begin{algorithmic}[1]
\Require next-action rule \(h:\cX\to\cA\), stopping predicate \(\halt:\cX\times\cA^*\to\{0,1\}\)
\Statex \hspace{\algorithmicindent} aligned state \(x=(e,w)\in\cX\)
\Ensure emitted suffix \(y\in\cA^*\), if the procedure terminates
\State \(y\gets\eps\)
\While{\(\halt(x,y)=0\)}
    \State \(a\gets h(x\cdot y)\)
    \State \(y\gets ya\)
\EndWhile
\State \Return \(y\)
\end{algorithmic}
\end{algorithm}

We will assume the following pointwise halting condition throughout the paper
whenever a rollout class is formed.  Without this condition,
\(\Roll_\halt(h)\) may fail to be a total function on \(\cX\), so
\(\Roll_\halt(\cH)\) is not a well-defined multiclass function class.

\begin{definition}[Pointwise halting]
\label{def:pointwise-halting}
Let \(\cE\) be an environment space, \(\cA\) an action space,
\(\cX=\cE\times\cA^*\), and let
\(\halt:\cX\times\cA^*\to\{0,1\}\) be a stopping predicate.  A next-action
class \(\cH\subseteq\cA^\cX\) is pointwise \(\halt\)-halting if, for every
\(h\in\cH\) and every initial state \(x\in\cX\), the rollout computation in
\Cref{alg:autoregressive-rollout} on input \((h,\halt,x)\) terminates after
finitely many iterations.
\end{definition}

\begin{definition}[Rollout class]
\label{def:rollout-class}
Let \(\cE\) be an environment space, \(\cA\) an action space,
\(\cX=\cE\times\cA^*\), and \(\halt:\cX\times\cA^*\to\{0,1\}\) a stopping
predicate.  Let \(\cH\subseteq\cA^\cX\) be a pointwise
\(\halt\)-halting next-action class.  Define \(\Roll_\halt(h):\cX\to\cA^*\)
by letting \(\Roll_\halt(h)(x)\) be the string returned by
\Cref{alg:autoregressive-rollout} on input \((h,\halt,x)\), and define the
rollout class
\(\Roll_\halt(\cH)=\{\Roll_\halt(h):h\in \cH\}\subseteq(\cA^*)^\cX\).
When \(\halt\) is fixed, we write \(\Roll(h)\) and \(\Roll(\cH)\).
\end{definition}

\begin{definition}[Exact-trace loss]
\label{def:exact-trace-loss}
Let \(\cE\), \(\cA\), \(\cX=\cE\times\cA^*\), a stopping predicate \(\halt:\cX\times\cA^*\to\{0,1\}\), and a pointwise \(\halt\)-halting class \(\cH\subseteq\cA^\cX\) be given.  A realizable example from target \(h^\star\in \cH\) is \((x,\Roll_\halt(h^\star)(x))\), where \(x\in \cX\).  A predictor \(G:\cX\to\cA^*\) makes a mistake at \(x\) if \(G(x)\ne \Roll_\halt(h^\star)(x)\).
Thus exact-trace learning of \(\cH\) is ordinary multiclass learning of \(\Roll_\halt(\cH)\).
\end{definition}

\begin{proposition}[Exact-trace learning is multiclass learning of the rollout class]
\label{prop:exact-trace-is-rollout}
Let \(\cE\), \(\cA\), \(\cX=\cE\times\cA^*\), a stopping predicate \(\halt:\cX\times\cA^*\to\{0,1\}\), and a pointwise \(\halt\)-halting class \(\cH\subseteq\cA^\cX\) be given.  The realizable exact-trace learning problem for \(\cH\) is the realizable multiclass learning problem for \(\Roll_\halt(\cH)\subseteq(\cA^*)^\cX\).
In particular, one-inclusion density and PAC sample complexity for exact-trace learning are the corresponding quantities for \(\Roll_\halt(\cH)\).
\end{proposition}

\begin{proof}[Proof of \Cref{prop:exact-trace-is-rollout}]
For a fixed target \(h^\star\in \cH\), every training example has the form
\[
        (x,\Roll_\halt(h^\star)(x)).
\]
This is exactly a realizable multiclass example for the target function \(\Roll_\halt(h^\star)\in\Roll_\halt(\cH)\).  Conversely, every target in
\(\Roll_\halt(\cH)\) is \(\Roll_\halt(h)\) for some \(h\in \cH\).  The loss in
\Cref{def:exact-trace-loss} is ordinary zero-one loss against that target function.
\end{proof}

\subsection{Multiclass density and learning black boxes}

\begin{definition}[Realizable PAC sample complexity]
\label{def:pac-sample-complexity}
Let \(\cH\subseteq\cY^\cX\) be a multiclass class and let \(\varepsilon,\delta\in(0,1)\).  A possibly randomized learner takes a finite labeled sample in \((\cX\times\cY)^m\) and returns a predictor \(G:\cX\to\cY\).  The realizable PAC sample complexity \(n_{\PAC}^{\varepsilon,\delta}(\cH)\) is the least \(m\ge0\) such that some learner has the following guarantee: for every target \(h^\star\in\cH\) and every distribution \(\mathcal D\) on \(\cX\), if \(x_1,\ldots,x_m\) are sampled independently from \(\mathcal D\) and the learner receives \(((x_i,h^\star(x_i)))_{i=1}^m\), then with probability at least \(1-\delta\), over the sample and the learner's randomness, the returned predictor \(G\) satisfies
\[
        \mathcal D\{x\in\cX:G(x)\ne h^\star(x)\}\le\varepsilon.
\]
If no such \(m\) exists, set \(n_{\PAC}^{\varepsilon,\delta}(\cH)=\infty\).
\end{definition}

\begin{definition}[Transductive leave-one-out error]
\label{def:loo-error}
Let \(\cH\subseteq\cY^\cX\) be a multiclass class and let \(n\ge1\).  A possibly randomized transductive leave-one-out rule \(A\) at sample size \(n\) receives a tuple
\(\mathbf x=(x_1,\ldots,x_n)\in\cX^n\), an index \(i\in[n]\), and the revealed labels \((y_j)_{j\ne i}\in\cY^{[n]\setminus\{i\}}\), and outputs a label in \(\cY\) for coordinate \(i\).  No distinctness is assumed among the points \(x_1,\ldots,x_n\).

The rule is permutation-invariant if, for every permutation of \([n]\), simultaneously permuting the tuple, the hidden index, and the revealed labels leaves the distribution of the output label unchanged.  Its leave-one-out error on \(\cH\) is
\[
        \sup_{\substack{h^\star\in\cH\\
                        \mathbf x=(x_1,\ldots,x_n)\in\cX^n}}
        \frac1n\sum_{i=1}^n
        \Pr\!\left[
        A(\mathbf x,i,(h^\star(x_j))_{j\ne i})\ne h^\star(x_i)
        \right],
\]
where the probability is over the rule's randomness.
\end{definition}

\begin{definition}[Density and one-inclusion density]
\label{def:finite-density}
Let \(D\) be a finite domain set, let \(\cY\) be a label set, and let \(\cG\subseteq \cY^D\) be finite and nonempty.  The density of \(\cG\) over \(D\) is
\[
        \dens_D(\cG)=
        |D|-\frac1{|\cG|}
        \sum_{x\in D}\left|\cG|_{D\setminus\{x\}}\right|
        =
        \frac1{|\cG|}\sum_{x\in D}
        \left(|\cG|-\left|\cG|_{D\setminus\{x\}}\right|\right).
\]
For a multiclass class \(\cH\subseteq \cY^\cX\) and a tuple \(\mathbf x=(x_1,\ldots,x_n)\in\cX^n\), write
\[
        \cH|_{\mathbf x}
        =
        \{(h(x_1),\ldots,h(x_n)):h\in\cH\}
        \subseteq\cY^{[n]}.
\]
For a multiclass class \(\cH\subseteq \cY^\cX\), define
\[
        \mu_\cH(n)=
        \sup_{\mathbf x\in\cX^n}
        \sup_{\substack{\emptyset\ne \cH'\subseteq \cH\\ \cH'\text{ finite}}}
        \dens_{[n]}(\cH'|_{\mathbf x}).
\]
If \(\cX^n=\emptyset\), the outer supremum is taken to be \(0\).
\end{definition}

\begin{definition}[Pseudo-cubes and DS dimension {\citep{DanielyShalev2014}}]
\label{def:ds-dimension}
Let \(D\) be a finite domain set.  A finite nonempty class \(Q\subseteq \cY^D\) is a pseudo-cube on \(D\) if for every \(q\in Q\) and every \(x\in D\), there exists \(q'\in Q\) such that
\[
        q'(x)\ne q(x),
        \qquad
        q'(y)=q(y)\quad\forall y\in D\setminus\{x\}.
\]
If \(|D|=d\), we also call \(Q\) a \(d\)-pseudo-cube.  A class \(\cH\subseteq \cY^\cX\) DS-shatters \(D\) if \(\cH|_D\) contains a pseudo-cube on \(D\).  Its DS dimension is denoted \(\DSdim(\cH)\), with value \(0\) if no positive-dimensional pseudo-cube exists and \(\infty\) if the dimensions are unbounded.
\end{definition}

We use two standard black boxes, stated below in the notation of this paper:
the randomized one-inclusion bound in transductive leave-one-out form and the
leave-one-out-to-PAC aggregation theorem.

\begin{theorem}[One-inclusion transductive leave-one-out bound {\citep{DanielyShalev2014,HannekeMengMoranShaeiri2026,Pabbaraju2026}}]
\label{thm:one-inclusion-loo}
Let \(\cH\subseteq \cY^\cX\) be a multiclass class, where \(\cY\) is arbitrary.  For every \(n\ge1\), there exists a possibly randomized permutation-invariant transductive leave-one-out rule at sample size \(n\) whose leave-one-out error, in the sense of \Cref{def:loo-error}, is at most
\[
        \frac{\lceil \mu_{\cH}(n)\rceil}{n},
\]
with the convention \(\lceil\infty\rceil=\infty\).  In particular, if
\(\mu_{\cH}(n)\le K\) for an integer \(K\), then the error is at most \(K/n\).
\end{theorem}

\begin{proof}[Proof sketch of \Cref{thm:one-inclusion-loo}]
Fix a tuple \(\mathbf x=(x_1,\ldots,x_n)\in\cX^n\), and consider the
one-inclusion hypergraph on the possibly infinite vertex set
\(\cH|_{\mathbf x}\subseteq\cY^{[n]}\).  Its hyperedges are the non-singleton
fibers of the coordinate-deletion maps
\[
        v\in\cH|_{\mathbf x}
        \longmapsto
        v|_{[n]\setminus\{i\}},
        \qquad i\in[n].
\]
For every finite vertex set \(W\subseteq\cH|_{\mathbf x}\), choose one
representative \(h\in\cH\) for each vertex of \(W\).  These representatives
form a finite subclass \(\cH_W\subseteq\cH\) with
\(\cH_W|_{\mathbf x}=W\).  Hence every finite induced subhypergraph has density
at most \(\mu_\cH(n)\), because \(\mu_\cH(n)\) takes the supremum over all
finite subclasses in \Cref{def:finite-density}.

The finite one-inclusion orientation theorem, in the density normalization of
\Cref{def:finite-density}, gives an orientation with maximum outdegree at most
\(\lceil\mu_\cH(n)\rceil\) on every finite induced subhypergraph; see
\citet{DanielyShalev2014} for the one-inclusion bound and
\citet[Appendix~A.1]{Pabbaraju2026} for this density normalization, building
on the orientation formulation of \citet{HannekeMengMoranShaeiri2026}.  By the
standard compactness argument for infinite one-inclusion graphs, these finite
orientations extend to an orientation of the full one-inclusion hypergraph;
see also \cite[Remark~4]{Pabbaraju2026}.

The resulting orientation defines a deterministic transductive leave-one-out
rule for the fixed ordered tuple \(\mathbf x\), with average leave-one-out
error at most \(\lceil\mu_\cH(n)\rceil/n\).  Finally, averaging this rule over
a uniformly random permutation of \([n]\) makes the rule permutation-invariant
without increasing its average leave-one-out error.
\end{proof}

\begin{theorem}[Leave-one-out to PAC aggregation {\citep{AdenAli2023}}]
\label{thm:loo-to-pac}
There is a universal constant \(C>0\) such that the following holds.  Let \(\cH\subseteq \cY^\cX\) be a multiclass class.  If, for some \(K\ge0\) and every \(n\ge1\), there is a permutation-invariant transductive leave-one-out rule at sample size \(n\) with leave-one-out error, in the sense of \Cref{def:loo-error}, at most \(K/n\), then for every \(\varepsilon,\delta\in(0,1)\), the PAC sample complexity of \Cref{def:pac-sample-complexity} satisfies
\[
        n_{\PAC}^{\varepsilon,\delta}(\cH)
        \le
        C\frac{K+\log(1/\delta)}{\varepsilon}.
\]
\end{theorem}

\section{Parity dimension and the main reduction}
\label{sec:parspan}

We now introduce the dimension parameter used in the proof.  By
\Cref{prop:exact-trace-is-rollout}, exact-trace learning has been reduced to
ordinary multiclass learning of the rollout class \(\Roll_\halt(\cH)\).  A
natural first attempt would be to compare
\(\DSdim(\Roll_\halt(\cH))\) with \(\DSdim(\cH)\).  This comparison is false:
\Cref{thm:ds-rollout-increase} shows that rollout can increase DS dimension.
Parity dimension is the replacement invariant.  It still controls multiclass
density, but it does not increase under autoregressive rollout.

The main reduction uses the chain
\[
        \mu_{\Roll_\halt(\cH)}(n)
        \le
        \ParDim(\Roll_\halt(\cH))
        \le
        \ParDim(\cH)
        \le
        \DSdim(\cH),
\]
where the first inequality is \Cref{thm:parity-density}, whose proof is deferred to \Cref{sec:finite-parspan-density}; the second is \Cref{thm:parity-rollout-nonincrease}; and the last is the elementary support argument included in \Cref{cor:pardim-ds}.

\subsection{Parity dimension}

\begin{definition}[Even pseudo-cube]
\label{def:even-pseudocube}
Let \(\cH\subseteq \cY^\cX\) be a function class, let \(D\subseteq\cX\) be finite, and let \(\cH'\subseteq\cH\) be finite and nonempty.  The restricted class \(\cH'|_D\subseteq\cY^D\) is an even pseudo-cube on \(D\) if there exists a nonzero vector \(\eta\in\Ftwo^{\cH'|_D}\) such that for every coordinate \(x\in D\) and every restricted pattern \(u\in\cH'|_{D\setminus\{x\}}\),
\[
        \sum_{\substack{v\in \cH'|_D:\\ v|_{D\setminus\{x\}}=u}}\eta(v)=0.
\]
If \(|D|=d\), we also call it a \(d\)-dimensional even pseudo-cube.
\end{definition}

\begin{definition}[Parity dimension]
\label{def:parity-dimension}
Let \(\cH\subseteq \cY^\cX\) be a function class.  The parity dimension \(\ParDim(\cH)\) is the supremum of all \(d\ge0\) such that \(\cH\) has a finite nonempty subclass \(\cH'\subseteq\cH\) and a finite domain set \(D\subseteq\cX\) with \(|D|=d\) for which \(\cH'|_D\) is an even pseudo-cube on \(D\).
If no positive-dimensional even pseudo-cube exists, we set
\(\ParDim(\cH)=0\); in particular, empty classes have parity dimension \(0\).
\end{definition}

\subsection{From parity dimension to the main theorem}

\begin{theorem}[Parity density theorem]
\label{thm:parity-density}
Let \(\cH\subseteq\cY^\cX\) be a function class.  For every \(n\ge1\),
\[
        \mu_\cH(n)\le\ParDim(\cH).
\]
\end{theorem}

The proof of \Cref{thm:parity-density} is deferred to
\Cref{sec:finite-parspan-density}; it is a finite-dimensional
linear-algebra argument on finite restrictions.

\begin{corollary}[Density, parity dimension, and DS dimension]
\label{cor:pardim-ds}
For every function class \(\cH\subseteq\cY^\cX\),
\[
        \frac12\DSdim(\cH)
        \le
        \sup_{n\ge1}\mu_\cH(n)
        \le
        \ParDim(\cH)
        \le
        \DSdim(\cH).
\]
For binary classes, \(\ParDim(\cH)=\DSdim(\cH)=\VCdim(\cH)\).
\end{corollary}

The three inequalities have separate sources.  The left inequality is the standard lower comparison between DS witnesses and one-inclusion density: a \(d\)-dimensional DS pseudo-cube has density at least \(d/2\).  The middle inequality is \Cref{thm:parity-density}.  The right inequality holds because the support of an even certificate is an ordinary DS pseudo-cube on the same domain.  For binary classes, \(\DSdim=\VCdim\), and a VC-shattered binary cube carries the all-one even certificate.  The complete proof is deferred to \Cref{app:dimension-comparisons}.

\begin{theorem}[Parity dimension does not increase under rollout]
\label{thm:parity-rollout-nonincrease}
Let \(\cE\) be an environment space, \(\cA\) an action space, \(\cX=\cE\times\cA^*\), and \(\halt:\cX\times\cA^*\to\{0,1\}\).  If \(\cH\subseteq\cA^\cX\) is pointwise \(\halt\)-halting, then \(\ParDim(\Roll_\halt(\cH))\le\ParDim(\cH)\).
\end{theorem}

\Cref{sec:parity-rollout-nonincrease} is devoted to proving \Cref{thm:parity-rollout-nonincrease}.  Its proof is the only place where prefix-tree structure of autoregressive rollouts is used.

\begin{theorem}[DS dimension can increase under rollout]
\label{thm:ds-rollout-increase}
There are finite sets \(\cE,\cA\), an aligned state space \(\cX=\cE\times\cA^*\), a stopping rule \(\halt:\cX\times\cA^*\to\{0,1\}\), and a finite next-action class \(\cH\subseteq\cA^\cX\) such that
\[
        \DSdim(\cH)=2,
        \qquad
        \DSdim(\Roll_\halt(\cH))\ge3.
\]
In particular, the analogue of \Cref{thm:parity-rollout-nonincrease} with \(\DSdim\) in place of \(\ParDim\) is false.
\end{theorem}

\begin{corollary}[Parity dimension and DS dimension can separate]
\label{cor:pardim-ds-separation}
For the example in \Cref{thm:ds-rollout-increase}, the rollout class satisfies
\[
        \ParDim(\Roll_\halt(\cH))=2,
        \qquad
        \DSdim(\Roll_\halt(\cH))\ge3.
\]
In particular, \(\ParDim\) can be strictly smaller than \(\DSdim\).
\end{corollary}

\begin{proof}[Proof of \Cref{cor:pardim-ds-separation}]
By \Cref{cor:pardim-ds} and \Cref{thm:ds-rollout-increase},
\[
        \ParDim(\cH)\le\DSdim(\cH)=2.
\]
Hence \Cref{thm:parity-rollout-nonincrease} gives
\[
        \ParDim(\Roll_\halt(\cH))\le2.
\]
On the other hand, \(\DSdim(\Roll_\halt(\cH))\ge3\) by
\Cref{thm:ds-rollout-increase}.  The lower comparison in \Cref{cor:pardim-ds}
then gives
\[
        \ParDim(\Roll_\halt(\cH))\ge\frac32.
\]
Since \(\ParDim\) is a supremum over integer dimensions and is at most \(2\)
here, this forces
\[
        \ParDim(\Roll_\halt(\cH))=2.
\]
\end{proof}

\Cref{sec:ds-blowup} gives the finite construction and verification.

We now prove the main theorem stated in the introduction.

\begin{proof}[Proof of \Cref{thm:pac-bound}]
\Cref{thm:parity-density} applied to the rollout class and \Cref{thm:parity-rollout-nonincrease} give
\[
        \mu_{\Roll_\halt(\cH)}(m)
        \le
        \ParDim(\Roll_\halt(\cH))
        \le
        \ParDim(\cH)
        \qquad\forall m\ge1.
\]
If \(\ParDim(\cH)=\infty\), the claimed upper bound is vacuous.  Otherwise,
\(\ParDim(\cH)\) is an integer, and \(\mu_{\Roll_\halt(\cH)}(m)\le\ParDim(\cH)\)
for every \(m\ge1\), so the ceiling in \Cref{thm:one-inclusion-loo} does not
change the bound.  Apply \Cref{thm:one-inclusion-loo} with
\(K=\ParDim(\cH)\), then apply \Cref{thm:loo-to-pac}.  The DS-dimension
version follows from \Cref{cor:pardim-ds}.
\end{proof}

\begin{corollary}[Worst-case optimality {\citep{DanielyShalev2014,Pabbaraju2026}}]
\label{cor:worst-case-optimality}
The dependence on \(\DSdim(\cH)\) in \Cref{thm:pac-bound} is worst-case optimal.  Indeed, if \(\halt(x,\eps)=0\) for every \(x\in\cX\) and \(\halt(x,a)=1\) for every \(x\in\cX\) and every \(a\in\cA\), viewing \(a\) as a length-one string in \(\cA^*\), then \(\Roll_\halt(\cH)\) is isomorphic to \(\cH\) with labels written as length-one strings.
\end{corollary}

\begin{proof}[Proof of \Cref{cor:worst-case-optimality}]
Under one-step stopping, \(\Roll_\halt(h)(x)=h(x)\) as a length-one action string.  Thus ordinary realizable multiclass learning of \(\cH\) embeds into exact-trace learning of \(\Roll_\halt(\cH)\).  The multiclass lower bound in terms of \(\DSdim(\cH)\) applies to this special case.
\end{proof}

\section{Finite low-coordinate spanning and density}
\label{sec:finite-parspan-density}

This section proves the finite-dimensional linear-algebra step behind
\Cref{thm:parity-density}.  The arguments take place on finite restrictions of
an arbitrary function class \(\cH\subseteq\cY^\cX\); the global class itself
need not be finite.  The key finite statement is that if no large restriction
\(\cH|_{D'}\) is an even pseudo-cube, then low-coordinate functions span all
functions on \(\cH|_D\).  A touch-counting argument then converts this spanning
statement into the desired density bound:
\[
        \mathsf L_k^D(\cH)=\Ftwo^{\cH|_D},
        \qquad
        \cH|_D\ne\emptyset\Longrightarrow \dens_D(\cH|_D)\le k.
\]
The proof is inspired by the polynomial spanning and touch-counting argument of
\citet{Pabbaraju2026}.  The conclusion here is a refinement of the usual
DS-dimension density theorem: Pabbaraju's bound follows from
\(\ParDim(\cF)\le\DSdim(\cF)\), while
\Cref{cor:pardim-ds-separation} shows that this refinement can be strict.  By
\Cref{cor:pardim-ds}, the sharper parameter is still within a factor \(2\) of
\(\DSdim\).

The nonemptiness condition is needed only when the density expression is
evaluated, because density divides by the size of the finite restriction.  The
argument is independent of the autoregressive structure used later.

\subsection{Marginals and low-coordinate duality}

We first set up the finite linear algebra.  Marginal maps push parity vectors
from a larger finite restriction to a smaller one.  Low-coordinate functions are
the dual objects: they are functions on \(\cH|_D\) that only inspect a small
coordinate subset.

\begin{definition}[Marginal maps]
\label{def:marginal-maps}
Let \(\cH\subseteq\cY^\cX\) be a function class.  For finite nested domain sets
\[
        D''\subseteq D'\subseteq\cX,
\]
such that \(\cH|_{D'}\) is finite, the marginalization map from \(D'\) to \(D''\) is the linear map
\[
        M_{D''}^{D'}[\cH]:
        \Ftwo^{\cH|_{D'}}
        \to
        \Ftwo^{\cH|_{D''}}
\]
defined by
\[
        \bigl(M_{D''}^{D'}[\cH]\eta\bigr)(u)
        =
        \sum_{\substack{v\in\cH|_{D'}:\\ v|_{D''}=u}}\eta(v),
        \qquad u\in\cH|_{D''}.
\]
\end{definition}

\begin{lemma}[Compositionality of marginalization]
\label{lem:marginalization-compositionality}
Let \(\cH\subseteq\cY^\cX\) be a function class.  Let \(D'''\subseteq D''\subseteq D'\subseteq\cX\) be finite domain sets such that \(\cH|_{D'}\) is finite.  Then
\[
        M_{D'''}^{D'}[\cH]
        =
        M_{D'''}^{D''}[\cH]M_{D''}^{D'}[\cH].
\]
\end{lemma}

\begin{proof}[Proof of \Cref{lem:marginalization-compositionality}]
For \(u\in\cH|_{D'''}\) and \(\eta\in\Ftwo^{\cH|_{D'}}\), the left-hand side is
\[
        \sum_{\substack{v\in\cH|_{D'}:\\ v|_{D'''}=u}}\eta(v).
\]
The right-hand side first groups each \(v\in\cH|_{D'}\) by its restriction to \(D''\), and then sums over exactly those \(D''\)-patterns whose restriction to \(D'''\) is \(u\).  It is therefore the same sum.
\end{proof}

\begin{definition}[Low-coordinate spaces and annihilators]
\label{def:low-coordinate-spans}
Let \(\cH\subseteq\cY^\cX\) be a function class, let \(D\subseteq\cX\) be finite with \(\cH|_D\) finite, and let \(D'\subseteq D\).  Define
\[
        J_{D'}^D(\cH)=
        \{f\in\Ftwo^{\cH|_D}:
        f(u)=f(v)\text{ whenever }u|_{D'}=v|_{D'}\}.
\]
Equivalently, \(J_{D'}^D(\cH)\) is the pullback of the full function space \(\Ftwo^{\cH|_{D'}}\) along the restriction map \(\cH|_D\to\cH|_{D'}\).  Thus
\[
        \dim J_{D'}^D(\cH)=\left|\cH|_{D'}\right|.
\]
For \(k\ge0\), define the low-coordinate span
\[
        \mathsf L_k^D(\cH)=
        \sum_{\substack{D'\subseteq D\\ |D'|\le k}}J_{D'}^D(\cH)
        \subseteq\Ftwo^{\cH|_D}.
\]
This is the space of all \(\Ftwo\)-valued functions on \(\cH|_D\) that can be written as sums of functions, each of which depends on at most \(k\) coordinates from \(D\).

Define the order-\(k\) marginal annihilator
\[
        \mathsf N_k^D(\cH)=
        \bigcap_{\substack{D'\subseteq D\\ |D'|\le k}}
        \ker M_{D'}^D[\cH]
        \subseteq\Ftwo^{\cH|_D}.
\]
This is the space of parity vectors whose marginal on every coordinate subset of size at most \(k\) is zero.
\end{definition}

\begin{lemma}[Low-coordinate duality]
\label{lem:low-coordinate-duality}
Let \(\cH\subseteq\cY^\cX\) be a function class, let \(D\subseteq\cX\) be finite with \(\cH|_D\) finite, and let \(k\ge0\).  Then
\[
        \bigl(\mathsf L_k^D(\cH)\bigr)^\perp=\mathsf N_k^D(\cH),
\]
under the pairing
\[
        \langle \eta,f\rangle=
        \sum_{u\in\cH|_D}\eta(u)f(u).
\]
\end{lemma}

\begin{proof}[Proof of \Cref{lem:low-coordinate-duality}]
For every \(D'\subseteq D\), a function \(f\in J_{D'}^D(\cH)\) has the form \(f=\phi\circ(u\mapsto u|_{D'})\) for some \(\phi:\cH|_{D'}\to\Ftwo\).  Therefore
\[
        \langle\eta,f\rangle
        =
        \sum_{w\in\cH|_{D'}}\phi(w)
        \bigl(M_{D'}^D[\cH]\eta\bigr)(w).
\]
Thus \(\eta\) annihilates \(J_{D'}^D(\cH)\) if and only if \(M_{D'}^D[\cH]\eta=0\).  Taking the orthogonal complement of the sum defining \(\mathsf L_k^D(\cH)\) gives the displayed identity.
\end{proof}

\subsection{From parity vanishing to spanning}

The next step translates absence of large even pseudo-cubes into vanishing of
high-order marginal annihilators, and then into full low-coordinate spanning.

\begin{lemma}[Top marginal annihilators are even certificates]
\label{lem:top-annihilator-even}
Let \(\cH\subseteq\cY^\cX\) be a function class, and let \(\emptyset\ne D\subseteq\cX\) be finite with \(\cH|_D\) finite.  Then \(\cH|_D\) is an even pseudo-cube on \(D\) if and only if
\[
        \mathsf N_{|D|-1}^D(\cH)\ne0.
\]
\end{lemma}

\begin{proof}[Proof of \Cref{lem:top-annihilator-even}]
If \(0\ne\eta\in \mathsf N_{|D|-1}^D(\cH)\), then in particular
\[
        M_{D\setminus\{x\}}^D[\cH]\eta=0
        \qquad\forall x\in D.
\]
These are precisely the coordinate-deletion marginal conditions in \Cref{def:even-pseudocube}, so \(\cH|_D\) is an even pseudo-cube on \(D\).

Conversely, suppose \(0\ne\eta\in\Ftwo^{\cH|_D}\) satisfies all coordinate-deletion marginal conditions:
\[
        M_{D\setminus\{x\}}^D[\cH]\eta=0
        \qquad\forall x\in D.
\]
For every proper subset \(D''\subsetneq D\), choose \(x\in D\setminus D''\).  By \Cref{lem:marginalization-compositionality},
\[
        M_{D''}^D[\cH]
        =
        M_{D''}^{D\setminus\{x\}}[\cH]
        M_{D\setminus\{x\}}^D[\cH]
\]
shows that \(M_{D''}^D[\cH]\eta=0\).  Hence \(\eta\in\mathsf N_{|D|-1}^D(\cH)\), so \(\mathsf N_{|D|-1}^D(\cH)\ne0\).
\end{proof}

\begin{lemma}[High-order parity vanishing gives low-coordinate spanning]
\label{lem:parspan-vanishing}
Let \(\cH\subseteq\cY^\cX\) be a function class, let \(D\subseteq\cX\) be finite with \(\cH|_D\) finite, and let \(k\ge0\).  Suppose that
\[
        \mathsf N_{|D'|-1}^{D'}(\cH)=0
        \qquad
        \forall D'\subseteq D\text{ with }|D'|\ge k+1.
\]
Then
\[
        \mathsf L_k^D(\cH)=\Ftwo^{\cH|_D}.
\]
\end{lemma}

\begin{proof}[Proof of \Cref{lem:parspan-vanishing}]
By \Cref{lem:low-coordinate-duality}, it suffices to show \(\mathsf N_k^D(\cH)=0\).  Let \(\eta\in \mathsf N_k^D(\cH)\).  If \(\eta=0\), there is nothing to prove.  Otherwise, since \(M_D^D[\cH]\eta=\eta\ne0\), choose an inclusion-minimal set \(D'\subseteq D\) such that
\[
        \theta:=M_{D'}^D[\cH]\eta\ne0.
\]
For every \(x\in D'\), minimality gives
\[
        M_{D'\setminus\{x\}}^D[\cH]\eta=0.
\]
By \Cref{lem:marginalization-compositionality},
\[
        M_{D'\setminus\{x\}}^{D'}[\cH]\theta
        =
        M_{D'\setminus\{x\}}^D[\cH]\eta
        =0.
\]
For every proper subset \(D''\subsetneq D'\), choose \(x\in D'\setminus D''\).  By the same factorization argument as in \Cref{lem:top-annihilator-even},
\[
        M_{D''}^{D'}[\cH]\theta
        =
        M_{D''}^{D'\setminus\{x\}}[\cH]
        M_{D'\setminus\{x\}}^{D'}[\cH]\theta
        =0.
\]
Thus
\[
        0\ne\theta\in \mathsf N_{|D'|-1}^{D'}(\cH).
\]
Moreover \(|D'|\ge k+1\), because \(|D'|\le k\) would force \(M_{D'}^D[\cH]\eta=0\) from \(\eta\in \mathsf N_k^D(\cH)\).  This contradicts the assumed high-order vanishing.  Therefore every \(\eta\in \mathsf N_k^D(\cH)\) is zero.  By \Cref{lem:low-coordinate-duality},
\[
        \bigl(\mathsf L_k^D(\cH)\bigr)^\perp=0.
\]
Since \(\Ftwo^{\cH|_D}\) is finite-dimensional, \(\mathsf L_k^D(\cH)=\Ftwo^{\cH|_D}\).
\end{proof}

\subsection{From spanning to density}

The final step turns low-coordinate spanning into a one-inclusion density bound
by counting how many coordinates are touched by a basis.

\begin{lemma}[Low-coordinate bases imply low density]
\label{lem:touch-counting-density}
Let \(\cH\subseteq\cY^\cX\) be a function class, let \(D\subseteq\cX\) be finite with \(\cH|_D\) finite, and let \(k\ge0\).  Assume \(\cH|_D\ne\emptyset\).  If
\[
        \mathsf L_k^D(\cH)=\Ftwo^{\cH|_D},
\]
then
\[
        \dens_D(\cH|_D)\le k.
\]
\end{lemma}

\begin{proof}[Proof of \Cref{lem:touch-counting-density}]
This is the same touch-counting mechanism as in the polynomial proof of the
multiclass density theorem by \citet{Pabbaraju2026}.  There, one chooses a
spanning family of low-support monomials in coordinate-value indicator
variables; each basis element touches only the coordinates appearing in that
monomial.  Here the basis elements are low-coordinate functions, and \(D_b\)
records the touched coordinate set for a basis vector \(b\).

Because \(\mathsf L_k^D(\cH)=\Ftwo^{\cH|_D}\), there exists a basis
\[
        \cB\subseteq
        \bigcup_{\substack{D'\subseteq D\\ |D'|\le k}}J_{D'}^D(\cH)
\]
of \(\Ftwo^{\cH|_D}\).  For every \(b\in\cB\), choose a set \(D_b\subseteq D\) such that
\[
        |D_b|\le k,
        \qquad
        b\in J_{D_b}^D(\cH).
\]
For every coordinate \(x\in D\), let
\[
        q_x:\Ftwo^{\cH|_D}\to
        \Ftwo^{\cH|_D}/J_{D\setminus\{x\}}^D(\cH)
\]
be the quotient map.  If \(x\notin D_b\), then \(D_b\subseteq D\setminus\{x\}\), so \(b\) does not depend on \(x\).  Equivalently,
\[
        b\in J_{D\setminus\{x\}}^D(\cH)
        \quad\text{and}\quad
        q_x(b)=0.
\]
The vectors \(q_x(b)\), \(b\in\cB\), span the quotient.  Therefore only basis vectors with \(x\in D_b\) can contribute to the quotient dimension, and
\[
        \codim J_{D\setminus\{x\}}^D(\cH)
        \le
        |\{b\in\cB:x\in D_b\}|.
\]
Since \(\dim J_{D\setminus\{x\}}^D(\cH)=\left|\cH|_{D\setminus\{x\}}\right|\), the left-hand side is \(\left|\cH|_D\right|-\left|\cH|_{D\setminus\{x\}}\right|\).  Using \Cref{def:finite-density} and summing the displayed quotient bound over \(x\in D\),
\[
\begin{aligned}
        \left|\cH|_D\right|\dens_D(\cH|_D)
        &=
        \sum_{x\in D}\left(\left|\cH|_D\right|-\left|\cH|_{D\setminus\{x\}}\right|\right)
        \le
        \sum_{x\in D}|\{b\in\cB:x\in D_b\}|\\
        &=
        \sum_{b\in\cB}|D_b|
        \le k|\cB|=k\left|\cH|_D\right|.
\end{aligned}
\]
Dividing by \(\left|\cH|_D\right|\) gives \(\dens_D(\cH|_D)\le k\).
\end{proof}

\begin{proof}[Proof of \Cref{thm:parity-density}]
Let \(\cF\subseteq\cY^\cX\) be a function class.  If \(\ParDim(\cF)=\infty\), there is nothing to prove.  Otherwise put
\[
        p=\ParDim(\cF)<\infty.
\]
Fix \(n\ge1\), a tuple \(\mathbf x=(x_1,\ldots,x_n)\in\cX^n\), and a finite nonempty subclass \(\cH\subseteq\cF\).  Put
\[
        \cG=\cH|_{\mathbf x}
        =
        \{(h(x_1),\ldots,h(x_n)):h\in\cH\}
        \subseteq\cY^{[n]}.
\]
We will prove \(\dens_{[n]}(\cG)\le p\).  Since \(\mathbf x\) and \(\cH\) were arbitrary, this proves the theorem by \Cref{def:finite-density}.

The tuple \(\mathbf x\) may have repeated entries, so we first show that any
high-order even certificate must live on distinct original domain points.

For every coordinate subset \(A\subseteq[n]\) with \(|A|\ge p+1\), the following high-order vanishing holds:
\[
        \mathsf N_{|A|-1}^{A}(\cG)=0.
\]
Indeed, suppose otherwise.  By \Cref{lem:top-annihilator-even}, \(\cG|_A\) is an even pseudo-cube on the coordinate set \(A\).  Let \(0\ne\eta\in\Ftwo^{\cG|_A}\) be an even certificate.

First, the original points \(x_i\), \(i\in A\), must be distinct.  If \(i\ne j\) in \(A\) and \(x_i=x_j\), then every pattern \(v\in\cG|_A\) satisfies \(v(i)=v(j)\).  Hence the coordinate-deletion map \(\cG|_A\to\cG|_{A\setminus\{i\}}\) is injective, because \(v(i)\) is recovered from the remaining coordinate \(j\).  Its fibers are singletons, so the deletion-marginal condition at coordinate \(i\) forces \(\eta(v)=0\) for every \(v\in\cG|_A\), contradicting \(\eta\ne0\).

Thus the map \(i\mapsto x_i\) is injective on \(A\).  Let \(D_A=\{x_i:i\in A\}\).  Relabeling each coordinate \(i\in A\) by the distinct domain point \(x_i\) turns \(\cG|_A\) into \(\cH|_{D_A}\), and preserves the coordinate-deletion marginal conditions.  Therefore \(\cH|_{D_A}\) is an even pseudo-cube on the domain \(D_A\).  Since \(|D_A|=|A|>p\) and \(\cH\subseteq\cF\), this contradicts the definition of \(p=\ParDim(\cF)\).

Apply \Cref{lem:parspan-vanishing} to the finite class \(\cG\subseteq\cY^{[n]}\) on the domain \([n]\).  The high-order vanishing gives
\[
        \mathsf L_p^{[n]}(\cG)=\Ftwo^\cG.
\]
Since \(\cG\ne\emptyset\), \Cref{lem:touch-counting-density} gives
\[
        \dens_{[n]}(\cG)\le p.
\]
Taking the supremum over all tuples \(\mathbf x\in\cX^n\) and all finite nonempty \(\cH\subseteq\cF\) gives
\[
        \mu_\cF(n)\le p=\ParDim(\cF)
        \qquad\forall n\ge1.
\]
\end{proof}

\section{Parity dimension does not increase under rollout}
\label{sec:partition-peeling}
\label{sec:parity-rollout-nonincrease}

This section proves \Cref{thm:parity-rollout-nonincrease}, namely
\[
        \ParDim(\Roll_\halt(\cH))\le \ParDim(\cH)
\]
for every pointwise \(\halt\)-halting next-action class
\(\cH\subseteq\cA^\cX\) over aligned states \(\cX=\cE\times\cA^*\).  This is
the only part of the proof that uses autoregressive prefix structure.  The
argument first proves a partition-tree peeling theorem, then applies it to
finite rollout trace trees, and finally converts the resulting trace-level
certificate into an even pseudo-cube for the local next-action class.

\subsection{Pure partition-tree peeling}

\begin{definition}[Partition tree]
\label{def:partition-tree}
Let \(L\) be a finite nonempty set.  A partition tree \(\cT\) on \(L\) is a finite rooted tree with a nonempty block \(\Blk(u)\subseteq L\) assigned to each node \(u\) such that:
\begin{enumerate}[label=(\roman*),leftmargin=2em]
\item the root \(\rho\) satisfies \(\Blk(\rho)=L\);
\item leaf nodes have singleton blocks;
\item if \(u\) is internal, then its children partition its block:
\[
        \Blk(u)=\bigsqcup_{c\in\Ch(u)}\Blk(c).
\]
\end{enumerate}
Unary chains and repeated blocks along a parent-child edge are allowed.
\end{definition}

\begin{lemma}[Zero-mass branching]
\label{lem:zero-mass-branching}
Let \(\cT\) be a partition tree on a finite set \(L\), let \(\Omega\) be a finite parameter set, and let \(\theta:L\times\Omega\to\Ftwo\).
Let \(u_0\) be a node such that
\[
        \sum_{\ell\in\Blk(u_0)}\theta(\ell,\omega)=0
        \qquad\forall\omega\in\Omega,
\]
and suppose that \(\theta\) is not identically zero on \(\Blk(u_0)\times\Omega\).  Then there exists an internal descendant \(u\) of \(u_0\), possibly \(u_0\) itself, such that
\[
        \sum_{\ell\in\Blk(u)}\theta(\ell,\omega)=0
        \qquad\forall\omega\in\Omega,
\]
and the child quotient
\[
        Q_u(c,\omega)=
        \sum_{\ell\in\Blk(c)}\theta(\ell,\omega),
        \qquad c\in\Ch(u),
\]
is nonzero.  Moreover,
\[
        \sum_{c\in\Ch(u)}Q_u(c,\omega)=0
        \qquad\forall\omega\in\Omega.
\]
\end{lemma}

\begin{proof}[Proof of \Cref{lem:zero-mass-branching}]
Among all descendants \(v\) of \(u_0\), including \(u_0\) itself, whose block has zero mass for every \(\omega\in\Omega\) and on whose block \(\theta\) is not identically zero, choose one with inclusion-minimal block.  Among nodes with that same block, choose \(v\) deepest.

The node \(v\) cannot be a leaf.  If \(\Blk(v)=\{\ell\}\), then zero mass gives \(\theta(\ell,\omega)=0\) for every \(\omega\in\Omega\), contrary to nonzero restriction.  Hence \(v\) is internal.

If the child quotient at \(v\) were zero, every child \(c\in\Ch(v)\) would have zero mass for every \(\omega\).  Since the children partition \(\Blk(v)\) and \(\theta\) is nonzero on \(\Blk(v)\times\Omega\), some child \(c_0\) has nonzero restriction.  If \(\Blk(c_0)\subsetneq\Blk(v)\), this contradicts inclusion-minimality.  If \(\Blk(c_0)=\Blk(v)\), this contradicts the deepest-node tie breaker.  Therefore the child quotient at \(v\) is nonzero.  Set \(u=v\).

The final displayed identity follows because the child blocks partition \(\Blk(u)\):
\[
        \sum_{c\in\Ch(u)}Q_u(c,\omega)
        =
        \sum_{\ell\in\Blk(u)}\theta(\ell,\omega)
        =
        0.
\]
\end{proof}

\begin{definition}[Block quotient tensor]
\label{def:block-quotient}
Let \(L_1,\ldots,L_d\) be finite sets and let \(\eta:L_1\times\cdots\times L_d\to\Ftwo\).
For subsets \(B_r\subseteq L_r\), define
\[
        \eta[B_1,\ldots,B_d]
        =
        \sum_{\ell_1\in B_1}\cdots\sum_{\ell_d\in B_d}
        \eta(\ell_1,\ldots,\ell_d).
\]
A block system in coordinate \(r\) is a finite family of pairwise disjoint node blocks in a partition tree on \(L_r\).  It need not partition all of \(L_r\); it partitions only its active subset.
\end{definition}

\begin{theorem}[Product partition peeling]
\label{thm:product-peeling}
For each \(r\in[d]\), let \(\cT_r\) be a partition tree on a finite set \(L_r\).  Let \(0\ne\eta:L_1\times\cdots\times L_d\to\Ftwo\) be line-even in every coordinate, meaning that for every \(r\in[d]\) and every fixing of the other coordinates,
\[
        \sum_{\ell_r\in L_r}
        \eta(\ell_1,\ldots,\ell_d)=0.
\]
Then there are internal nodes \(u_r\in\cT_r\), \(r\in[d]\), such that the child quotient
\[
        Q(c_1,\ldots,c_d)=
        \eta[\Blk(c_1),\ldots,\Blk(c_d)],
        \qquad c_r\in\Ch(u_r),
\]
is nonzero and line-even in every coordinate.
\end{theorem}

\begin{proof}[Proof of \Cref{thm:product-peeling}]
We use block systems in the sense of \Cref{def:block-quotient}.  Start with the singleton block system in each coordinate.  The corresponding quotient tensor is \(\eta\), so it is nonzero and line-even.  Invariant: after a set \(S\subseteq[d]\) of coordinates has been processed, the quotient tensor over the child-block systems in coordinates \(S\) and singleton systems elsewhere is nonzero and line-even in every coordinate.

Process coordinates one at a time.  Suppose the current quotient tensor is nonzero and line-even, and coordinate \(r\) is still represented by singleton leaf blocks.  Treat all other current block choices as a finite parameter set \(\Omega\).  For \(\ell\in L_r\) and parameter \(\omega\), let \(\theta(\ell,\omega)\) be the current quotient entry obtained by using \(\{\ell\}\) in coordinate \(r\) and the blocks specified by \(\omega\) in all other coordinates.

Nonzeroness of the current quotient says that \(\theta\) is not identically zero.  Line-evenness in coordinate \(r\) says that
\[
        \sum_{\ell\in L_r}\theta(\ell,\omega)=0
        \qquad\forall\omega\in\Omega.
\]
Apply \Cref{lem:zero-mass-branching} to the root of \(\cT_r\).  It returns an internal node \(u_r\) whose child quotient is nonzero and child-even.  Replace the singleton block system in coordinate \(r\) by the child blocks \(\{\Blk(c):c\in\Ch(u_r)\}\).  The active subset in coordinate \(r\) becomes \(\Blk(u_r)\); later quotient tensors are taken over the product of the active block systems.

The new quotient is nonzero by the nonzero child quotient.  It is line-even in coordinate \(r\) by the child-even conclusion of \Cref{lem:zero-mass-branching}.  For a different coordinate \(q\), the new line sum is a finite sum of old line sums, one for each leaf in a child block \(\Blk(c)\).  Each old line sum is zero by the induction invariant.  Thus line-evenness in all other coordinates is preserved.

After all coordinates are processed, the quotient tensor has the displayed form and is nonzero and line-even.
\end{proof}

\subsection{From trace parity to base next-action parity}

\begin{lemma}[Prefix trace sets are prefix-free]
\label{lem:prefix-free}
Let \(\cE\) be an environment space, \(\cA\) an action space, \(\cX=\cE\times\cA^*\), and \(\halt:\cX\times\cA^*\to\{0,1\}\).  Let \(\cH\subseteq\cA^\cX\) be pointwise \(\halt\)-halting.  Fix a start state \(x\in \cX\) and a finite subset \(\cH'\subseteq \cH\).  Then the finite set \(L_x=\{\Roll_\halt(h)(x):h\in \cH'\}\subseteq\cA^*\) is prefix-free.
\end{lemma}

\begin{proof}[Proof of \Cref{lem:prefix-free}]
Suppose \(u,v\in L_x\) and \(u\) is a proper prefix of \(v\).  The rollout that emits exactly \(u\) stops after emitted suffix \(u\), so \(\halt(x,u)=1\).  The rollout that emits \(v\) has the same initial state \(x\) and the same emitted prefix \(u\), but must continue beyond \(u\), so \(\halt(x,u)=0\).  This is impossible.
\end{proof}

\begin{definition}[Prefix partition tree]
\label{def:prefix-partition-tree}
Let \(L\subseteq\cA^*\) be finite and prefix-free.  Its prefix partition tree has one node for every prefix of a string in \(L\).  The block at prefix \(u\) is \(\Blk(u)=\{\ell\in L:u\text{ is a prefix of }\ell\}\).
The children of \(u\) are the nonempty strings \(ua\) with \(a\in\cA\) such that \(ua\) is a prefix of some string in \(L\).
Since \(L\) is prefix-free, these prefix blocks satisfy the partition-tree axioms of \Cref{def:partition-tree}; in particular, children partition a prefix block by the unique next action after the prefix.
\end{definition}

\begin{theorem}[Trace parity peels to base parity]
\label{thm:trace-to-base}
Let \(\cE\) be an environment space, \(\cA\) an action space, \(\cX=\cE\times\cA^*\), and \(\halt:\cX\times\cA^*\to\{0,1\}\).  Let \(\cH\subseteq\cA^\cX\) be pointwise \(\halt\)-halting.  Fix states \(x_1,\ldots,x_d\in \cX\) and a finite subset \(\cH'\subseteq \cH\).  Let
\[
        V\subseteq L_1\times\cdots\times L_d,
        \qquad
        L_r=\{\Roll_\halt(h)(x_r):h\in\cH'\},
\]
be the finite set of distinct rollout trace tuples realized by \(\cH'\) on \(x_1,\ldots,x_d\).  View \(V\) as a finite class on the indexed domain \([d]\), with \(v(r)\in L_r\).  For every
\[
        0\ne\eta\in\Ftwo^V
        \quad\text{such that}\quad
        M_{[d]\setminus\{r\}}^{[d]}[V]\eta=0
        \quad\forall r\in[d],
\]
there exist prefixes \(u_r\in\cA^*\) such that the base next-action value table of the finite subclass \(\cH'\) on the indexed states
\[
        x_r\cdot u_r,
        \qquad r=1,\ldots,d,
\]
carries a nonzero parity certificate in all \(d\) indexed coordinate directions.
\end{theorem}

\begin{proof}[Proof of \Cref{thm:trace-to-base}]
Extend \(\eta\) by zero to a tensor on \(L_1\times\cdots\times L_d\).  This extended tensor is nonzero and line-even in every coordinate because each coordinate-deletion marginal of \(\eta\) is zero: for a fixed coordinate direction \(r\), each full line in \(L_1\times\cdots\times L_d\) intersects \(V\) in one fiber of the restriction map \(V\to V|_{[d]\setminus\{r\}}\), and \(\eta\) has zero sum on every such fiber.

By \Cref{lem:prefix-free}, each \(L_r\) is prefix-free, so it has a prefix partition tree.  Apply \Cref{thm:product-peeling}.  We get internal prefix nodes \(u_r\) such that
\[
        Q(c_1,\ldots,c_d)
        =
        \eta[\Blk(c_1),\ldots,\Blk(c_d)],
        \qquad c_r\in\Ch(u_r),
\]
is nonzero and line-even.

A child \(c_r\in\Ch(u_r)\) has the form \(u_ra_r\) for a unique action \(a_r\in\cA\).  Relabel the child quotient by these actions.  Since different children of a prefix correspond to different next actions, this relabeling preserves nonzeroness and line-evenness.

If a relabeled entry \((a_1,\ldots,a_d)\) has coefficient \(1\), then
\[
        \sum_{v\in V\cap(\Blk(u_1a_1)\times\cdots\times\Blk(u_da_d))}
        \eta(v)
        =
        1 .
\]
In particular, at least one such tuple has \(\eta(v)=1\), and hence at least one representative \(h\in \cH'\) realizes it.  For that \(h\), the rollout from \(x_r\) passes through prefix \(u_r\) and then appends \(a_r\), so
\[
        h(x_r\cdot u_r)=a_r.
\]
Let
\[
        V_{\mathrm{base}}
        =
        \{(h(x_1\cdot u_1),\ldots,h(x_d\cdot u_d)):h\in\cH'\}
        \subseteq \cA^{[d]} .
\]
Define a vector \(\xi\in\Ftwo^{V_{\mathrm{base}}}\) by setting
\(\xi(a_1,\ldots,a_d)=Q(a_1,\ldots,a_d)\) on base patterns arising from the relabeled child quotient, and \(\xi=0\) on all other base patterns.  The preceding paragraph shows that every support point of \(Q\) is realized in \(V_{\mathrm{base}}\), so \(\xi\) is nonzero.  Moreover, the line-evenness of \(Q\) implies that every coordinate-deletion marginal of \(\xi\) is zero: if the fixed \((d-1)\)-pattern uses an action outside the corresponding child sets, all coefficients in that fiber are zero; otherwise the marginal is exactly one of the line sums of \(Q\).  Hence \(\xi\) is a nonzero even certificate for the indexed base value table.
\end{proof}

\subsection{Parity dimension does not increase under rollout}

\begin{proof}[Proof of \Cref{thm:parity-rollout-nonincrease}]
Take a finite witness to \(\ParDim(\Roll_\halt(\cH))\).  Thus there are a finite state set \(D=\{x_1,\ldots,x_d\}\subseteq \cX\) with \(|D|=d\), and a finite nonempty subclass \(U\subseteq\Roll_\halt(\cH)\), such that \(U|_D\subseteq(\cA^*)^D\) is a \(d\)-dimensional even pseudo-cube.  Equivalently, after enumerating \(D\) as \(x_1,\ldots,x_d\), the trace table \(V=U|_D\), viewed as a finite class on the indexed domain \([d]\), carries a nonzero vector \(0\ne\eta\in\Ftwo^V\) whose coordinate-deletion marginals vanish:
\[
        M_{[d]\setminus\{r\}}^{[d]}[V]\eta=0
        \qquad\forall r\in[d].
\]
For every rollout function in \(U\), choose one representative hypothesis in \(\cH\) that realizes it; these representatives form a finite set \(\cH'\subseteq\cH\).

Apply \Cref{thm:trace-to-base} to \(\cH'\), the states \(x_1,\ldots,x_d\), and the certificate \(\eta\).  We obtain prefixes \(u_1,\ldots,u_d\in\cA^*\) and a nonzero parity certificate for the base value table of \(\cH'\) on the indexed states
\[
        \widetilde x_r=x_r\cdot u_r,
        \qquad r=1,\ldots,d.
\]
The states \(\widetilde x_1,\ldots,\widetilde x_d\) are distinct.  Indeed, if \(\widetilde x_r=\widetilde x_s\) for some \(r\ne s\), then every base value pattern has identical coordinates \(r\) and \(s\).  The restriction map deleting coordinate \(r\) is then injective on the indexed base table, so \(M_{[d]\setminus\{r\}}^{[d]}\) has zero kernel on that table.  This contradicts the existence of a nonzero certificate in all \(d\) indexed coordinate directions.

Thus \(\widetilde D=\{\widetilde x_1,\ldots,\widetilde x_d\}\) is a finite domain set of size \(d\), and \(\cH'|_{\widetilde D}\) is a \(d\)-dimensional even pseudo-cube.  Hence \(\ParDim(\cH)\ge d\).  Taking the supremum over all rollout witnesses proves the theorem.
\end{proof}

\section{Application to full-label multi-instance learning}
\label{sec:full-label-mil-application}

We now spell out the full-label multi-instance learning consequence in its
original sequence-lift language.  The result follows from
\Cref{thm:pac-bound} by encoding the entire local-instance list in the
environment and letting the rollout emit the local labels one by one.

\paragraph{We use the following notation.}
Let \(\cU\) be a local instance space and let \(\cY\) be a nonempty local label
space.  Write
\[
        \cU^\star=\bigsqcup_{T\ge1}\cU^T,
        \qquad
        \cY^\star=\bigsqcup_{T\ge1}\cY^T
\]
for nonempty finite lists of local instances and labels.  For
\(U=(u_1,\ldots,u_T)\in\cU^\star\), write \(|U|=T\).  A local rule
\(g:\cU\to\cY\) induces the full-label list rule
\[
        g^\star:\cU^\star\to\cY^\star,
        \qquad
        g^\star(u_1,\ldots,u_T)=(g(u_1),\ldots,g(u_T)).
\]
For a local class \(\cG\subseteq\cY^\cU\), define its full-label lift
\[
        \cG^\star=\{g^\star:g\in\cG\}\subseteq(\cY^\star)^{\cU^\star}.
\]
Learning \(\cG^\star\) is ordinary realizable multiclass learning whose
examples are finite lists and whose labels are full local-label lists.  The
loss is all-or-nothing: a prediction on \(U=(u_1,\ldots,u_T)\) is correct only
if all \(T\) local labels are correct.
The data distribution in this PAC problem is an arbitrary distribution on
\(\cU^\star\).  In particular, the entries inside a list need not be
independent, identically distributed, exchangeable, or of fixed length.

\paragraph{Relation to bag-label MIL.}
Classical multiple-instance learning studies bags with weak or aggregate labels
rather than the full local-label list; see \citet{MaronLozanoPerez1998} and
\citet{SabatoTishby2012}.  The present application is different.  The input is
an ordered finite list, positions need not be exchangeable, and supervision
reveals the complete label vector \((g(u_1),\ldots,g(u_T))\).  Thus the
question is not how to infer instance labels from a bag label, but whether
all-or-nothing correctness of the full label list costs more samples than
ordinary local multiclass learning.

\begin{theorem}[Optimal PAC sample complexity of full-label multi-instance learning]
\label{thm:full-label-mil}
For every local multiclass class \(\cG\subseteq\cY^\cU\) over a nonempty label
space \(\cY\), the following class-by-class statement holds.  If
\(\DSdim(\cG)\ge1\), then for every
\(\varepsilon,\delta\in(0,1)\),
\[
        n_{\PAC}^{\varepsilon,\delta}(\cG^\star)
        =
        \Theta\!\left(n_{\PAC}^{\varepsilon,\delta}(\cG)\right)
        =
        \Theta\!\left(
        \frac{\DSdim(\cG)+\log(1/\delta)}{\varepsilon}
        \right).
\]
If \(\DSdim(\cG)=0\), then \(\cG^\star\) has a single labeling on every finite
list sample, and \(n_{\PAC}^{\varepsilon,\delta}(\cG^\star)=0\).
\end{theorem}

\begin{proof}[Proof of \Cref{thm:full-label-mil}]
Fix \(y_0\in\cY\).  Rephrase the full-label problem as an autoregressive
rollout.  Let
\[
        \cE=\cU^\star,
        \qquad
        \cA=\cY,
        \qquad
        \cX=\cE\times\cY^*.
\]
For \(g\in\cG\), define \(h_g:\cX\to\cY\) by
\[
        h_g((u_1,\ldots,u_T),w)
        =
        \begin{cases}
        g(u_{|w|+1}), & |w|<T,\\
        y_0, & |w|\ge T.
        \end{cases}
\]
Let \(\cH_\cG=\{h_g:g\in\cG\}\).  Use the stopping predicate
\[
        \halt((U,w),y)=1
        \quad\Longleftrightarrow\quad
        |y|\ge |U|.
\]
The class \(\cH_\cG\) is pointwise \(\halt\)-halting, and the rollout from
\((U,\eps)\) emits exactly \(g^\star(U)\).  Hence \(\cG^\star\) is the
restriction of \(\Roll_\halt(\cH_\cG)\) to the initial-state slice
\(\cU^\star\times\{\eps\}\).

It remains to compare the local dimensions.  We claim that
\(\ParDim(\cH_\cG)=\ParDim(\cG)\).  The inequality
\(\ParDim(\cH_\cG)\ge\ParDim(\cG)\) follows from singleton lists: for every
local instance \(u\in\cU\),
\[
        h_g((u),\eps)=g(u).
\]
For the reverse inequality, take any finite domain set
\(D\subseteq\cU^\star\times\cY^*\).  A state \(((u_1,\ldots,u_T),w)\in D\)
with \(|w|<T\) queries the local coordinate \(u_{|w|+1}\); a state with
\(|w|\ge T\) is constant equal to \(y_0\).  If an even certificate on
\(\cH_\cG|_D\) used a constant coordinate, then the marginal deleting that
coordinate would be injective and would have zero kernel.  The certificate
would be zero, a contradiction.  Likewise, if two coordinates in \(D\) queried
the same local instance \(u\), then deleting one of them would leave the other
one, so the deletion marginal would again be injective.  Therefore every
nonzero even certificate on \(\cH_\cG|_D\) uses distinct nonconstant local
coordinates, and the same vector is an even certificate for the restriction of
\(\cG\) to those local coordinates.  Thus \(\ParDim(\cH_\cG)\le\ParDim(\cG)\).

Since \(\cG^\star\) is the restriction of \(\Roll_\halt(\cH_\cG)\) to the
initial-state slice \(\cU^\star\times\{\eps\}\), any learner for
\(\Roll_\halt(\cH_\cG)\) gives a learner for \(\cG^\star\).  By
\Cref{thm:pac-bound} and \(\ParDim(\cG)\le\DSdim(\cG)\),
\[
        n_{\PAC}^{\varepsilon,\delta}(\cG^\star)
        \le
        O\!\left(
        \frac{\DSdim(\cG)+\log(1/\delta)}{\varepsilon}
        \right).
\]
The lower bound is the singleton-list embedding.  Any local example
\((u,g(u))\) is the full-label example \(((u),(g(u)))\), so learning
\(\cG^\star\) is at least as hard as learning this fixed local class
\(\cG\).  The lower-bound half of the standard sharp realizable multiclass
PAC theorem, due to \citet{DanielyShalev2014}, \citet{Brukhim2022}, and
\citet{Pabbaraju2026}, gives
\[
        n_{\PAC}^{\varepsilon,\delta}(\cG)
        \ge
        \Omega\!\left(
        \frac{\DSdim(\cG)+\log(1/\delta)}{\varepsilon}
        \right)
\]
whenever \(\DSdim(\cG)\ge1\).  Together with the preceding upper bound, this
gives the displayed \(\Theta\)-rate.

If \(\DSdim(\cG)=0\), then no one-point domain is DS-shattered.  Hence all
functions in \(\cG\) agree on every \(u\in\cU\), so all functions in
\(\cG^\star\) agree on every finite list \(U\in\cU^\star\).  The realizable
problem has no label uncertainty, and \(n_{\PAC}^{\varepsilon,\delta}
(\cG^\star)=0\).
\end{proof}

\paragraph{A random ERM baseline.}
The preceding theorem is an optimal sample-complexity statement; it should not
be read as saying that every natural proper ERM tie-breaking rule is optimal.
For a finite local class \(\cG\) and a full-label sample
\[
        S=((U_i,Y_i))_{i=1}^n,
\]
let
\[
        V_\cG(S)=
        \{g\in\cG:g^\star(U_i)=Y_i\text{ for all }i\in[n]\}
\]
be the proper version space.  The uniform random proper ERM draws \(g\)
uniformly from \(V_\cG(S)\) whenever this set is finite and nonempty, and
outputs \(g^\star\).

\begin{theorem}[Uniform random ERM can be logarithmically suboptimal]
\label{thm:random-erm-mil-lower}
There is a universal constant \(c>0\) such that the following holds.  For
every \(d\ge1\), every \(T_{\max}\ge2\), and every
\(\varepsilon\in(0,1/16]\), set
\[
        q=\lfloor\log_2 T_{\max}\rfloor .
\]
There exist a finite instance space \(\cU\), a finite binary class
\(\cG\subseteq\{0,1\}^{\cU}\) with
\[
        \VCdim(\cG)=\DSdim(\cG)=d,
\]
and a realizable distribution \(\mathcal D\) over full-label list examples, supported
on lists of length at most \(T_{\max}\), such that the uniform random proper
ERM trained on \(n\) iid examples satisfies
\[
        n\le c\,\frac{dq}{\varepsilon}
        \quad\Longrightarrow\quad
        \Pr\!\left[
        \operatorname{err}_{\mathcal D}(\widehat g^\star)>\varepsilon
        \right]\ge \frac18 .
\]
Thus this random ERM rule can require
\[
        \Omega\!\left(
        \frac{d\log T_{\max}}{\varepsilon}
        \right)
\]
samples at constant confidence, even though
\Cref{thm:full-label-mil} gives an optimal learner with no dependence on
\(T_{\max}\).
\end{theorem}

\begin{proof}[Proof of \Cref{thm:random-erm-mil-lower}]
See \Cref{app:random-erm-mil-lower}.
\end{proof}

\section{Discussion}
\label{sec:discussion}

\paragraph{Expressiveness versus learnability.}
The empirical success of scratchpads, Chain-of-Thought prompting, zero-shot
Chain-of-Thought, self-consistency, and reasoning-action agents motivates the
study of complete intermediate traces
\citep{NyeEtAl2021Scratchpads,WeiEtAl2022CoT,KojimaEtAl2022ZeroShotCoT,
WangEtAl2022SelfConsistency,Yao2023ReAct}.  A complementary theoretical
line studies why intermediate generation can increase the expressive or
computational power of models.  For transformers, Merrill and Sabharwal show
that the amount of intermediate generation changes the computational class that
can be recognized, while Li, Liu, Zhou, and Ma show that Chain-of-Thought
enables constant-depth transformers to perform inherently sequential
computation
\citep{MerrillSabharwal2023CoTExpressivePower,LiLiuZhouMa2024CoTSequential}.
Recent autoregressive or recursive reasoning models, such as PENCIL and
recursive models, further explore ways of organizing long reasoning traces
through reduction, memory reuse, or recursive subproblem calls
\citep{YangSrebroMcAllesterLi2025PENCIL,YangSrebroLi2026RecursiveModels}.
These works ask why intermediate tokens, actions, reductions, or recursive
calls can make a model more powerful.  Our question is orthogonal: assuming the
full trace is supervised and generated by a shared local rule, how many
examples are needed to learn the trace under all-or-nothing exact-trace loss?
\Cref{thm:pac-bound} shows that, in the realizable setting, the answer is
governed by the same local PAC complexity as ordinary next-action learning.  In
particular, the theorem is not tied to text CoT: any deterministic
autoregressive trace formalism that can be represented as a local rule together
with a stopping rule falls under the same sample-complexity bound.

\paragraph{A stronger rule-valued viewpoint.}
\Cref{sec:full-label-mil-application} treats the label-valued full-sequence
lift \(\cG^\star\), where the learner predicts only the realized label sequence
\((g(u_1),\ldots,g(u_T))\).  One can ask for a stronger representation-level
version in which the learner outputs local predictors, or a local rule whose
evaluation produces the full label sequence.  The compression route of
\citet{HannekeMehalelMoran2026} has this flavor: its reconstruction builds a
next-token rule from inflated trace samples and then rolls out this rule to
obtain the full trace predictor.  Thus their method can be viewed as an
improper rule-valued route to full-trace learning, rather than a purely
trace-label argument.

This perspective helps explain the stronger complexity parameters appearing in
that route.  It also suggests a natural open problem: can such rule-valued or
function-valued sequence lifts be learned at the optimal local rate governed
only by \(\DSdim(\cG)\)?  A positive answer, under an appropriate
finite-alphabet or multiclass formalization, would imply the label-valued
full-sequence theorem by evaluation, and would similarly imply the
autoregressive exact-trace theorem by evaluating the predicted local rule along
the rollout states.

\section{Conclusion}
\label{sec:conclusion}

We have shown that, under realizable full-trace supervision, autoregressive
exact-trace learning has no intrinsic sample-complexity penalty for trace
length: the optimal PAC rate is governed by the local next-action class.  The
proof identifies parity dimension as the rollout-stable quantity mediating
between one-inclusion density and DS dimension, while the counterexample in
\Cref{sec:ds-blowup} shows that DS dimension itself is not the invariant
preserved by rollout.  This leaves several
natural extensions beyond the present deterministic realizable setting,
including noisy or agnostic trace supervision, partial-trace feedback,
randomized policies, and computationally efficient learning algorithms.

\section*{Acknowledgments}
The author thanks ChatGPT for assistance in finding the finite example in
\Cref{sec:ds-blowup}; the construction and verification are the author's
responsibility.
This work is supported by DARPA AIQ under Agreement No. HR00112520023, NSF CAREER Award 2544658 and OpenAI Superalignment Fast Grant.

\appendix
\section{Dimension comparisons}
\label{app:dimension-comparisons}

For completeness, we prove the dimension comparisons used in
\Cref{cor:pardim-ds}.  Keeping the proof here lets the main text use the
comparison chain without interrupting the proof of the autoregressive PAC
bound.

\begin{proof}[Proof of \Cref{cor:pardim-ds}]
The upper density bound
\[
        \sup_{n\ge1}\mu_\cH(n)
        \le
        \ParDim(\cH)
\]
is exactly \Cref{thm:parity-density}.

Next, \(\ParDim(\cH)\le\DSdim(\cH)\) follows directly from the definitions.  Suppose \(D\subseteq\cX\) is finite, \(\cH'\subseteq\cH\) is finite and nonempty, and \(\cH'|_D\) is an even pseudo-cube on \(D\).  Let \(0\ne\eta\in\Ftwo^{\cH'|_D}\) be an even certificate, and let
\[
        Q=\supp(\eta)\subseteq\cH'|_D.
\]
For every \(q\in Q\) and every \(x\in D\), the coordinate-deletion fiber of \(\cH'|_D\to\cH'|_{D\setminus\{x\}}\) containing \(q\) has zero \(\eta\)-sum.  Since \(\eta(q)=1\), that fiber contains another point \(q'\in Q\).  Then \(q'\) agrees with \(q\) on \(D\setminus\{x\}\) and differs at \(x\).  Thus \(Q\) is a pseudo-cube on \(D\).

It remains to prove the lower density comparison
\[
        \frac12\DSdim(\cH)\le \sup_{n\ge1}\mu_\cH(n).
\]
This comparison is recalled in \cite[Remarks~1 and~3]{Pabbaraju2026}; we give the elementary proof in the present normalization.  Let \(D\subseteq\cX\) have size \(d\ge1\), and let \(Q\subseteq\cH|_D\) be a finite pseudo-cube on \(D\).  Choose a finite subclass \(\cH'\subseteq\cH\) with \(\cH'|_D=Q\).  Order \(D\) as \(D=\{x_1,\ldots,x_d\}\) and put \(\mathbf x=(x_1,\ldots,x_d)\).  Then \(\cH'|_{\mathbf x}\subseteq\cY^{[d]}\) is obtained from \(Q\subseteq\cY^D\) by relabeling the coordinate \(x_i\) as \(i\), so
\[
        \mu_\cH(d)\ge \dens_{[d]}(\cH'|_{\mathbf x})=\dens_D(Q).
\]
It suffices to prove
\[
        \dens_D(Q)\ge \frac d2.
\]
Fix \(x\in D\).  Every fiber of the restriction map
\[
        Q\to Q|_{D\setminus\{x\}}
\]
has size at least \(2\), because the pseudo-cube property supplies, for every point in the fiber, another point with the same \(D\setminus\{x\}\)-restriction and a different value at \(x\).  Hence
\[
        \left|Q|_{D\setminus\{x\}}\right|
        \le
        \frac{|Q|}{2}.
\]
Using \Cref{def:finite-density},
\[
        \dens_D(Q)
        =
        \frac1{|Q|}\sum_{x\in D}
        \left(|Q|-\left|Q|_{D\setminus\{x\}}\right|\right)
        \ge
        \frac d2.
\]
Taking the supremum over all positive-dimensional DS witnesses gives the lower density comparison.

For binary classes, \(\DSdim(\cH)=\VCdim(\cH)\) by the usual equivalence between binary pseudo-cubes and VC-shattered sets.  Conversely, if \(D\) is VC-shattered by a binary class, then the full binary cube \(\{0,1\}^D\) has the all-one vector as an even certificate: every coordinate-deletion fiber has exactly two points, so its \(\Ftwo\)-sum is zero.  Hence \(\ParDim(\cH)\ge\VCdim(\cH)\), while the already proved inequality \(\ParDim(\cH)\le\DSdim(\cH)=\VCdim(\cH)\) gives equality.
\end{proof}

\section{Uniform random ERM lower bound}
\label{app:random-erm-mil-lower}

This appendix proves \Cref{thm:random-erm-mil-lower}.  The construction is a
full-label multi-instance version of a missing-mass lower bound for uniform
random proper ERM.  It does not contradict \Cref{thm:full-label-mil}, which is
an existence theorem for an optimal learner rather than a guarantee for every
proper ERM tie-breaking rule.

\begin{proof}[Proof of \Cref{thm:random-erm-mil-lower}]
Let \(q=\lfloor\log_2 T_{\max}\rfloor\), so \(q\ge1\), and let
\[
        \cU=\{\star\}\sqcup\{(j,\sigma):j\in[d],\ \sigma\in\{0,1\}^q\}.
\]
For
\[
        a=(a_1,\ldots,a_d)\in(\{0,1\}^q\cup\{\bot\})^d,
\]
define \(g_a:\cU\to\{0,1\}\) by
\[
        g_a(j,\sigma)=1
        \quad\Longleftrightarrow\quad
        a_j=\sigma,
        \qquad
        g_a(\star)=0.
\]
Let
\[
        \cG=\{g_a:a\in(\{0,1\}^q\cup\{\bot\})^d\}.
\]
The all-zero target is \(g_\bot=g_{(\bot,\ldots,\bot)}\).

We first check the dimension.  The class \(\cG\) shatters the \(d\)-point set
\[
        \{(1,0^q),\ldots,(d,0^q)\},
\]
because coordinate \(j\) is labeled one by setting \(a_j=0^q\) and labeled
zero by setting \(a_j=\bot\).  Conversely, no set of size \(d+1\) is
VC-shattered.  If the set contains \(\star\), then \(\star\) is always labeled
zero.  Otherwise, by the pigeonhole principle, two selected points have the
same block index \(j\).  For a fixed \(j\), every \(g_a\) labels at most one
point \((j,\sigma)\) by one, so the all-ones labeling on those two points is
impossible.  Thus \(\VCdim(\cG)=d\), and
\(\DSdim(\cG)=d\) because the class is binary.

For every \(j\in[d]\) and \(r\in[q]\), define the rare list
\[
        U_{j,r}=((j,\sigma):\sigma_r=1),
\]
with an arbitrary fixed ordering.  Its length is \(2^{q-1}\le T_{\max}\).
Also define the dummy list \(U_0=(\star)\).  All labels below are generated by
the target \(g_\bot\), so every observed full-label list is all zero.

Let \(N=dq\) and put \(\rho=8\varepsilon\).  Define a realizable distribution
\(\mathcal D\) by
\[
        \Pr_{\mathcal D}[U=U_{j,r}]=\frac{\rho}{N}
        \quad (j\in[d],\ r\in[q]),
        \qquad
        \Pr_{\mathcal D}[U=U_0]=1-\rho,
\]
with labels \(g_\bot^\star(U)\).  Since
\(\varepsilon\le1/16\), this is a probability distribution.

Let \(S\sim\mathcal D^n\), and let \(\widehat g\) be drawn uniformly from the proper
version space \(V_\cG(S)\).  The target \(g_\bot\) is always in the version
space, so the version space is nonempty.  Conditional on \(S\), consider a rare
list \(U_{j,r}\) that did not appear in the sample.  Let \(B_j(S)\subseteq[q]\)
be the set of rare-list indices \(s\) for which \(U_{j,s}\) did appear in
\(S\).  A spike \(a_j\in\{0,1\}^q\) is consistent with the observed all-zero
lists in block \(j\) if and only if
\[
        (a_j)_s=0
        \qquad\forall s\in B_j(S).
\]
Among the consistent choices for the \(j\)-th component, there is also the
zero choice \(a_j=\bot\).  Since \(r\notin B_j(S)\), if
\[
        u=q-|B_j(S)|\ge1
\]
is the number of unobserved bit coordinates in block \(j\), then the number of
consistent spike choices is \(2^u\), and exactly \(2^{u-1}\) of them have
\((a_j)_r=1\).  Such choices put a one somewhere in \(U_{j,r}\), and therefore
make a full-label mistake on \(U_{j,r}\).  Hence, conditional on \(S\),
\[
        \Pr_{\widehat g\sim V_\cG(S)}
        [\widehat g^\star(U_{j,r})\ne g_\bot^\star(U_{j,r})]
        =
        \frac{2^{u-1}}{1+2^u}
        \ge
        \frac13 .
\]

Taking expectation over the sample and the random ERM draw gives
\[
        \mathbb E[
        \operatorname{err}_{\mathcal D}(\widehat g^\star)]
        \ge
        \frac{\rho}{3}
        \left(1-\frac{\rho}{N}\right)^n.
\]
Indeed, each rare list has mass \(\rho/N\), and the probability that a fixed
rare list is absent from \(n\) iid samples is
\((1-\rho/N)^n\).

If
\[
        n\le \frac{N}{4\rho},
\]
then
\[
        \left(1-\frac{\rho}{N}\right)^n
        \ge
        1-\frac{\rho n}{N}
        \ge
        \frac34,
\]
and therefore
\[
        \mathbb E[
        \operatorname{err}_{\mathcal D}(\widehat g^\star)]
        \ge
        \frac{\rho}{4}
        =
        2\varepsilon.
\]
The error of \(\widehat g^\star\) is supported only on the rare lists, whose
total mass is \(\rho=8\varepsilon\).  Hence
\[
        2\varepsilon
        \le
        \mathbb E[
        \operatorname{err}_{\mathcal D}(\widehat g^\star)]
        \le
        \varepsilon\,
        \Pr[\operatorname{err}_{\mathcal D}(\widehat g^\star)\le\varepsilon]
        +
        8\varepsilon\,
        \Pr[\operatorname{err}_{\mathcal D}(\widehat g^\star)>\varepsilon].
\]
Thus
\[
        \Pr[\operatorname{err}_{\mathcal D}(\widehat g^\star)>\varepsilon]
        \ge
        \frac17
        \ge
        \frac18.
\]
Since \(N=dq\), the implication holds with \(c=1/32\).  This proves the
theorem.
\end{proof}

\section{Online comparison}
\label{app:online-comparison}

This appendix records, in the notation of this paper, the online comparison
implicit in the autoregressive online-learning results of
\citet{DoronAradMehalelMossel2026} and
\citet{BalcanBlumFragkiaLiSharma2026}.  It is not used in the proof of the PAC
upper bound; it is included for completeness and to explain the online row of
\Cref{tab:cot-rates}.  In this setup, a complete-trace online mistake can be
charged to one local next-action mistake.  The simulation below is
oracle-assisted: it uses a halting oracle to decide whether the current local
prediction map terminates from the queried start state.  Thus this appendix is
only an information-theoretic online comparison and is not used in the PAC
proof in the main text.

\begin{definition}[Online mistake bound]
\label{def:online-mistake-bound}
For a class \(\cF\subseteq\cY^\cX\), let \(\Monline(\cF)\) be the smallest
integer \(M\) such that some deterministic full-information online learning
rule makes at most \(M\) mistakes on every finite realizable online stream for
\(\cF\).  If no such \(M\) exists, set \(\Monline(\cF)=\infty\).
\end{definition}

\begin{theorem}[Halting-oracle online rollout simulation]
\label{thm:online-rollout-simulation}
Let \(\cE\) be an environment space, \(\cA\) an action space,
\(\cX=\cE\times\cA^*\), and let
\(\halt:\cX\times\cA^*\to\{0,1\}\) be a stopping predicate.  If
\(\cH\subseteq\cA^\cX\) is pointwise \(\halt\)-halting and the local class
\(\cH\) has a deterministic full-information online learner with mistake bound
\(M\), then there is a halting-oracle-assisted full-information online learner
for \(\Roll_\halt(\cH)\) with mistake bound at most \(M\).
\end{theorem}

\begin{proof}[Proof of \Cref{thm:online-rollout-simulation}]
Let \(\mathsf{Alg}\) be an online learning rule for the local class \(\cH\)
with mistake bound \(M\).  After the current local history, let
\(\widehat h:\cX\to\cA\) denote the current prediction map induced by
\(\mathsf{Alg}\): on a local instance \(v\in\cX\), \(\widehat h(v)\) is the
label that \(\mathsf{Alg}\) would currently predict at \(v\).

On a trace-level round with start state \(x\in\cX\), use a halting oracle to
decide whether the deterministic rollout of \(\widehat h\) from \(x\) under
\(\halt\) terminates.  If it terminates, simulate it until termination and
predict the finite trace
\[
        \widehat y=\Roll_\halt(\widehat h)(x).
\]
If it does not terminate, predict \(\widehat y=\eps\).

After the true trace \(y=\Roll_\halt(h^\star)(x)\) is revealed, do nothing if
\(\widehat y=y\).  Suppose first that the oracle said that \(\widehat h\)
terminates and the finite prediction \(\widehat y\) is wrong.  Then
\(\widehat y\) and \(y\) are both terminal traces from the same start state
\(x\).  They cannot be strict prefixes of each other: if the shorter string has
already halted, then the same deterministic stopping predicate would also stop
the other rollout at the same emitted suffix.  Hence there is a first position
at which the two action strings differ.  Let \(u\in\cA^*\) be their common
prefix before that first disagreement.  The current local predictor makes a
local mistake at the aligned state \(x\cdot u\):
\[
        \widehat h(x\cdot u)\ne h^\star(x\cdot u).
\]
Feed the labeled local example \((x\cdot u,h^\star(x\cdot u))\) to
\(\mathsf{Alg}\).

It remains to handle the case where the oracle said that \(\widehat h\) does
not terminate, so the trace learner predicted \(\eps\), and this prediction is
wrong.  Then \(y\ne\eps\): if \(y=\eps\), then \(\halt(x,\eps)=1\), so every
rollout from \(x\), including the rollout of \(\widehat h\), would terminate
immediately.  Write the nonempty true trace as \(y=a_1\cdots a_T\).  For
\(t\in[T]\), let \(u_t=a_1\cdots a_{t-1}\) be the prefix before the \(t\)-th
true action.  There must be a first \(t\) such that
\[
        \widehat h(x\cdot u_t)\ne h^\star(x\cdot u_t).
\]
Indeed, if \(\widehat h\) agreed with \(h^\star\) at all prefix states
\(x\cdot u_t\) along the true trace, then \(\widehat h\) would emit the entire
string \(y\).  Since \(\halt(x,y)=1\), it would then terminate after \(y\),
contradicting the halting oracle's nontermination answer.  Feed the labeled
local example \((x\cdot u_t,h^\star(x\cdot u_t))\) to \(\mathsf{Alg}\).

Every complete-trace online mistake causes exactly one genuine local online
mistake update by the current local prediction map.  Therefore the
halting-oracle-assisted complete-trace online learner makes at most \(M\)
mistakes on every realizable online stream, proving the claim.
\end{proof}

\begin{theorem}[Multiclass Littlestone theorem {\citep{Littlestone1988,DanielySabatoBenDavidShalevShwartz2015}}]
\label{thm:multiclass-littlestone}
For every finite-label class \(\cF\subseteq\cY^\cX\), the optimal
full-information online mistake bound equals the multiclass Littlestone
dimension:
\[
        \Monline(\cF)=\Ldim(\cF).
\]
\end{theorem}

Combining \Cref{thm:multiclass-littlestone} with
\Cref{thm:online-rollout-simulation} gives an information-theoretic online
comparison by allowing a halting oracle.  In ordinary algorithmic settings the
same simulation can be implemented under additional assumptions, such as a
bounded horizon or proper predictors whose rollouts are guaranteed to halt.  In
finite-label or bounded-horizon settings where the standard multiclass online
mistake-bound theorem and online-to-batch conversion are applied, this gives
the online route summarized in \Cref{tab:cot-rates}.

\section{DS dimension is not rollout-stable}
\label{sec:ds-blowup}

This appendix gives the finite example promised in
\Cref{thm:ds-rollout-increase}: an aligned next-action class satisfying
\[
        \DSdim(\cH)=2
        \quad\text{but}\quad
        \DSdim(\Roll_\halt(\cH))\ge3.
\]
The example is not used in the proof of the PAC upper bound; it explains why
the proof uses parity dimension rather than DS dimension itself.
As recorded in \Cref{cor:pardim-ds-separation}, the same rollout class also
separates the two dimensions: its parity dimension is \(2\), while its
DS dimension is at least \(3\).

\begin{lemma}[Elementary pseudo-cube tests]
\label{lem:pseudocube-tests}
The following facts hold.
\begin{enumerate}[label=(\roman*)]
\item If \(Q\subseteq \cY^D\) is a pseudo-cube on \(D\) and \(B\subseteq D\), then the projected set
\[
        Q|_B=\{q|_B:q\in Q\}
\]
is a pseudo-cube on \(B\).
\item Let \(Q\subseteq L\times R\times C\) be a \(3\)-pseudo-cube.  For every \(\alpha\in L\), if
\[
        Q_\alpha=\{(r,c)\in R\times C:(\alpha,r,c)\in Q\}
\]
is nonempty, then \(Q_\alpha\) is a \(2\)-pseudo-cube on the row--column coordinates.
\item Let \(S\subseteq R\times C\) be finite.  Form the bipartite row--column graph with left vertex set \(R\), right vertex set \(C\), and edge set \(S\).  If this graph is a forest, then \(S\) contains no \(2\)-pseudo-cube.
\end{enumerate}
\end{lemma}

\begin{proof}[Proof of \Cref{lem:pseudocube-tests}]
For (i), fix \(u=q|_B\in Q|_B\) and a coordinate \(x\in B\).  Since \(Q\) is a pseudo-cube, there exists \(q'\in Q\) such that \(q'(x)\ne q(x)\) and \(q'\) agrees with \(q\) on \(D\setminus\{x\}\).  Then \(q'|_B\) differs from \(u\) at \(x\) and agrees with \(u\) on \(B\setminus\{x\}\).

For (ii), fix \((r,c)\in Q_\alpha\).  Since \(Q\) is a \(3\)-pseudo-cube, the point \((\alpha,r,c)\) has a flip in the \(R\)-coordinate and a flip in the \(C\)-coordinate while the \(L\)-coordinate remains \(\alpha\).  Thus \(Q_\alpha\) has a row flip and a column flip at \((r,c)\).  This is exactly the \(2\)-pseudo-cube condition on \(R\times C\).

For (iii), a \(2\)-pseudo-cube \(Q\subseteq S\) would be a nonempty edge set such that every edge has another edge sharing its row and another edge sharing its column.  Hence every vertex incident to an edge of \(Q\) has degree at least two in the subgraph induced by \(Q\).  A finite graph with minimum degree at least two contains a cycle, contradicting that the ambient row--column graph is a forest.
\end{proof}

\begin{proof}[Proof of \Cref{thm:ds-rollout-increase}]
\medskip\noindent\emph{Construction.}
Let
\[
\begin{aligned}
\cA_{\mathrm{sgn}}&=\{\mathsf P,\mathsf N\},&
\cA_{\mathsf P}&=\{\mathsf P_1,\mathsf P_2\},&
\cA_{\mathsf N}&=\{\mathsf N_1,\mathsf N_2\},\\
\cA_{\mathrm{row}}&=\{R_0,R_1,R_2\},&
\cA_{\mathrm{col}}&=\{C_0,C_1,C_2,C_3\},&
\cA_{\mathrm{off}}&=\{\bot\}.
\end{aligned}
\]
The action alphabet is the disjoint union
\[
        \cA=
        \cA_{\mathrm{sgn}}\sqcup\cA_{\mathsf P}\sqcup\cA_{\mathsf N}
        \sqcup\cA_{\mathrm{row}}\sqcup\cA_{\mathrm{col}}
        \sqcup\cA_{\mathrm{off}}.
\]
The action \(\bot\) is an ordinary off-branch action.  It is not a stopping token, and the root rollouts used below never emit it.
Let \(\cE=\{1,2,3\}\), let \(\cX=\cE\times\cA^*\), and use the stopping predicate
\[
\begin{aligned}
        \halt((e,w),y)=0
        \quad\Longleftrightarrow\quad&
        \bigl(w=\eps,\ e=1,\ |y|<2\bigr)\\
        &\text{or }
        \bigl(w=\eps,\ e\in\{2,3\},\ |y|<1\bigr).
\end{aligned}
\]
Thus only rollouts started from root states \(w=\eps\) emit new actions; non-root start states halt immediately.
Define a sign map
\[
        \beta(\mathsf P_1)=\beta(\mathsf P_2)=\mathsf P,
        \qquad
        \beta(\mathsf N_1)=\beta(\mathsf N_2)=\mathsf N.
\]

We first define the index set of hypotheses.  For \(r\in\{0,1,2\}\) and \(s\in\{0,1,2,3\}\), let \(T_{r,s}\subseteq\{\mathsf P_1,\mathsf P_2,\mathsf N_1,\mathsf N_2\}\) be given by
\[
\begin{array}{c|cccc}
 & s=0&s=1&s=2&s=3\\ \hline
r=0&\mathsf N_1\mathsf N_2&
      \mathsf P_1\mathsf P_2\mathsf N_1\mathsf N_2&
      \mathsf P_1\mathsf P_2&
      \mathsf P_2\mathsf N_2\\
r=1&\mathsf P_1\mathsf P_2&
      \mathsf N_1\mathsf N_2&
      \mathsf P_1\mathsf N_2&
      \mathsf P_2\mathsf N_1\\
r=2&\mathsf P_1\mathsf P_2\mathsf N_1\mathsf N_2&
      \mathsf P_1\mathsf P_2&
      \mathsf P_2\mathsf N_2&
      \mathsf N_1\mathsf N_2.
\end{array}
\]
For every triple \((\ell,r,s)\) with \(\ell\in T_{r,s}\), define a next-action rule
\[
        h_{\ell,r,s}:\cX\to\cA.
\]
The class is
\[
        \cH=\{h_{\ell,r,s}:\ell\in T_{r,s}\}.
\]
Thus \(\cH\) has \(28\) hypotheses.

The rule \(h_{\ell,r,s}\) is defined as follows.  In environment \(1\), it first emits the block \(\beta(\ell)\), then emits \(\ell\), and then the stopping predicate halts:
\[
        h_{\ell,r,s}(1,\eps)=\beta(\ell),
\]
\[
        h_{\ell,r,s}(1,\mathsf P)=
        \begin{cases}
        \ell,&\beta(\ell)=\mathsf P,\\
        \bot,&\beta(\ell)=\mathsf N,
        \end{cases}
        \qquad
        h_{\ell,r,s}(1,\mathsf N)=
        \begin{cases}
        \bot,&\beta(\ell)=\mathsf P,\\
        \ell,&\beta(\ell)=\mathsf N.
        \end{cases}
\]
In environments \(2\) and \(3\), it emits the row and column labels:
\[
        h_{\ell,r,s}(2,\eps)=R_r,
\]
\[
        h_{\ell,r,s}(3,\eps)=C_s.
\]
All unspecified states output \(\bot\).  The stopping rule above ensures pointwise halting, and the root rollouts used below never query a state whose returned action is \(\bot\).  Therefore
\[
        \Roll_\halt(h_{\ell,r,s})(1,\eps)=\beta(\ell)\ell,
        \qquad
        \Roll_\halt(h_{\ell,r,s})(2,\eps)=R_r,
        \qquad
        \Roll_\halt(h_{\ell,r,s})(3,\eps)=C_s.
\]
The construction separates what a rollout can see from what a single base query can see.  In environment \(1\), the rollout sees both the sign \(\beta(\ell)\) and then the refined token \(\ell\); a base query can ask only one of the states \((1,\eps),(1,\mathsf P),(1,\mathsf N)\) at a time.  Environments \(2\) and \(3\) reveal the row \(r\) and column \(s\).

The intended reading is a two-level sign code.  The symbols \(\mathsf P,\mathsf N\) are coarse signs:
\[
        \mathsf P\text{-block}=\{\mathsf P_1,\mathsf P_2\},
        \qquad
        \mathsf N\text{-block}=\{\mathsf N_1,\mathsf N_2\}.
\]
The state \((1,\eps)\) reveals only the coarse sign.  The state \((1,\mathsf P)\) reveals the refined token only inside the \(\mathsf P\)-block and returns the off-branch action \(\bot\) on the \(\mathsf N\)-block; the state \((1,\mathsf N)\) does the symmetric thing.  Thus a single base query sees only one layer of this two-level code, while a rollout from environment \(1\) first sees the sign and then follows the matching branch to recover the refined token.

The table \(T_{r,s}\) is chosen to have two opposite behaviors.  For the rollout lower bound, the full table
\[
        P=\{(\ell,r,s):\ell\in T_{r,s}\}
\]
has enough local redundancy: every cell contains at least two refined tokens, and for each fixed refined token \(\ell\), its row-column support has no row or column leaf.  This gives flips in the refined-token, row, and column coordinates.  For the base upper bound, every possible single-layer view is acyclic.  More precisely, the mixed-sign cells, the cells containing both \(\mathsf P_1,\mathsf P_2\), and the cells containing both \(\mathsf N_1,\mathsf N_2\) are all forests as row-column graphs.  Any hypothetical base \(3\)-pseudo-cube must place one first-coordinate slice inside one of these forests, where a row-column \(2\)-pseudo-cube cannot live.  The table is therefore a masking gadget: rollout unmasks a \(3\)-dimensional pseudo-cube, but every base view remains only \(2\)-dimensional.

\medskip\noindent\emph{The rollout class has a \(3\)-pseudo-cube.}
We prove first that \(\DSdim(\Roll_\halt(\cH))\ge3\).  Let
\[
        x_1=(1,\eps),
        \qquad
        x_2=(2,\eps),
        \qquad
        x_3=(3,\eps).
\]
The restriction of the rollout class to \(\{x_1,x_2,x_3\}\) contains, after injectively relabeling the three coordinates, the set
\[
        P=\{(\ell,r,s):\ell\in T_{r,s}\}
        \subseteq \{\mathsf P_1,\mathsf P_2,\mathsf N_1,\mathsf N_2\}\times\{0,1,2\}\times\{0,1,2,3\}.
\]
We claim that \(P\) is a \(3\)-pseudo-cube.  Every cell \(T_{r,s}\) has size at least two, so every point \((\ell,r,s)\in P\) has a flip in the refined-token coordinate.  For flips in the other two coordinates, write
\[
        G_\ell=\{(r,s):\ell\in T_{r,s}\}.
\]
From the displayed table,
\[
\begin{aligned}
G_{\mathsf P_1}&=\{(0,1),(0,2),(1,0),(1,2),(2,0),(2,1)\},\\
G_{\mathsf P_2}&=\{(0,1),(0,2),(0,3),(1,0),(1,3),(2,0),(2,1),(2,2)\},\\
G_{\mathsf N_1}&=\{(0,0),(0,1),(1,1),(1,3),(2,0),(2,3)\},\\
G_{\mathsf N_2}&=\{(0,0),(0,1),(0,3),(1,1),(1,2),(2,0),(2,2),(2,3)\}.
\end{aligned}
\]
In each \(G_\ell\), every row that appears has degree at least two and every column that appears has degree at least two.  Hence each \((\ell,r,s)\in P\) has an \(r\)-coordinate flip and an \(s\)-coordinate flip.  Thus \(P\) is a \(3\)-pseudo-cube, and
\[
        \DSdim(\Roll_\halt(\cH))\ge3.
\]

\medskip\noindent\emph{The base class has a \(2\)-pseudo-cube.}
Next we show \(\DSdim(\cH)=2\).  For the lower bound, restrict \(\cH\) to the two base states
\[
        x_0=(1,\eps),
        \qquad
        x_R=(2,\eps).
\]
The four hypotheses
\[
        h_{\mathsf P_1,0,1},\quad h_{\mathsf P_1,1,0},\quad
        h_{\mathsf N_1,0,0},\quad h_{\mathsf N_1,1,1}
\]
realize the four value pairs
\[
        (\mathsf P,R_0),\quad (\mathsf P,R_1),\quad
        (\mathsf N,R_0),\quad (\mathsf N,R_1).
\]
These form a \(2\)-pseudo-cube, so \(\DSdim(\cH)\ge2\).

\medskip\noindent\emph{No base \(3\)-pseudo-cube.}
It remains to prove \(\DSdim(\cH)\le2\).  The only nonconstant base states are
\[
        x_0=(1,\eps),
        \qquad
        x_{\mathsf P}=(1,\mathsf P),
        \qquad
        x_{\mathsf N}=(1,\mathsf N),
        \qquad
        x_R=(2,\eps),
        \qquad
        x_C=(3,\eps).
\]
A positive-dimensional pseudo-cube cannot use a constant coordinate, so every three-dimensional base witness would have to use three states among these five.  Moreover, by \Cref{lem:pseudocube-tests}(i), any base pseudo-cube of dimension at least \(3\) restricts to a \(3\)-pseudo-cube on any three of its coordinates, so ruling out three-state witnesses also rules out every witness of dimension greater than three.  It therefore suffices to show that no three of these five states carry a \(3\)-pseudo-cube.

The row-column obstruction used below is the following.  After relabeling the two row-column coordinates \(x_R,x_C\) by \((r,s)\), let \(E\) be the mixed-sign cells, let \(F_{\mathsf P}\) be the cells containing both \(\mathsf P_1\) and \(\mathsf P_2\), and let \(F_{\mathsf N}\) be the cells containing both \(\mathsf N_1\) and \(\mathsf N_2\).  Explicitly,
\[
\begin{aligned}
E&=\{(0,1),(0,3),(1,2),(1,3),(2,0),(2,2)\},\\
F_{\mathsf P}&=\{(0,1),(0,2),(1,0),(2,0),(2,1)\},\\
F_{\mathsf N}&=\{(0,0),(0,1),(1,1),(2,0),(2,3)\}.
\end{aligned}
\]
Equivalently, the three projections have the following row-column tables, where a bullet marks an included cell:
\begin{center}
\captionsetup{type=table,font=scriptsize}
\begin{minipage}[t]{0.31\linewidth}
\centering
{\scriptsize
\[
\setlength{\arraycolsep}{2.4pt}
\begin{array}{c|cccc}
E & s=0&s=1&s=2&s=3\\ \hline
r=0& &\bullet& &\bullet\\
r=1& & &\bullet&\bullet\\
r=2&\bullet& &\bullet&
\end{array}
\]
\captionof{table}{Mixed-sign cells \(E\): cells supporting a sign flip or a \(\bot\)-slice flip with \((r,s)\) fixed.}
\label{tab:counterexample-mixed-sign}
}
\end{minipage}
\hfill
\begin{minipage}[t]{0.31\linewidth}
\centering
{\scriptsize
\[
\setlength{\arraycolsep}{2.4pt}
\begin{array}{c|cccc}
F_{\mathsf P} & s=0&s=1&s=2&s=3\\ \hline
r=0& &\bullet&\bullet&\\
r=1&\bullet& & &\\
r=2&\bullet&\bullet& &
\end{array}
\]
\captionof{table}{Within-\(\mathsf P\) cells \(F_{\mathsf P}\): cells supporting a non-\(\bot\) \(x_{\mathsf P}\)-flip.}
\label{tab:counterexample-p-block}
}
\end{minipage}
\hfill
\begin{minipage}[t]{0.31\linewidth}
\centering
{\scriptsize
\[
\setlength{\arraycolsep}{2.4pt}
\begin{array}{c|cccc}
F_{\mathsf N} & s=0&s=1&s=2&s=3\\ \hline
r=0&\bullet&\bullet& &\\
r=1& &\bullet& &\\
r=2&\bullet& & &\bullet
\end{array}
\]
\captionof{table}{Within-\(\mathsf N\) cells \(F_{\mathsf N}\): cells supporting a non-\(\bot\) \(x_{\mathsf N}\)-flip.}
\label{tab:counterexample-n-block}
}
\end{minipage}
\end{center}
The row-column graphs in \Cref{tab:counterexample-mixed-sign,tab:counterexample-p-block,tab:counterexample-n-block} are forests: \(E\) is the path
\[
        s=1\;--\;r=0\;--\;s=3\;--\;r=1\;--\;s=2\;--\;r=2\;--\;s=0,
\]
\(F_{\mathsf P}\) is the path
\[
        s=2\;--\;r=0\;--\;s=1\;--\;r=2\;--\;s=0\;--\;r=1,
\]
and \(F_{\mathsf N}\) is the path
\[
        r=1\;--\;s=1\;--\;r=0\;--\;s=0\;--\;r=2\;--\;s=3.
\]
By \Cref{lem:pseudocube-tests}, none of the three obstruction graphs in \Cref{tab:counterexample-mixed-sign,tab:counterexample-p-block,tab:counterexample-n-block} contains a \(2\)-pseudo-cube.  The remaining proof only uses this fact.

First, no such witness can use two of \(x_0,x_{\mathsf P},x_{\mathsf N}\).  Indeed, the projections to the three pairs
\[
        (x_0,x_{\mathsf P}),\qquad (x_0,x_{\mathsf N}),\qquad (x_{\mathsf P},x_{\mathsf N})
\]
are respectively
\[
        \{(\mathsf P,\mathsf P_1),(\mathsf P,\mathsf P_2),(\mathsf N,\bot)\},
\]
\[
        \{(\mathsf P,\bot),(\mathsf N,\mathsf N_1),(\mathsf N,\mathsf N_2)\},
\]
and
\[
        \{(\mathsf P_1,\bot),(\mathsf P_2,\bot),(\bot,\mathsf N_1),(\bot,\mathsf N_2)\}.
\]
Each of these three row--column graphs is a forest, so by \Cref{lem:pseudocube-tests} none contains a \(2\)-pseudo-cube.  Since projections of pseudo-cubes are pseudo-cubes by \Cref{lem:pseudocube-tests}, no \(3\)-pseudo-cube can use two of \(x_0,x_{\mathsf P},x_{\mathsf N}\).  Hence a three-dimensional base witness must use one of
\[
        \{x_0,x_R,x_C\},
        \qquad
        \{x_{\mathsf P},x_R,x_C\},
        \qquad
        \{x_{\mathsf N},x_R,x_C\}.
\]
We rule out these three cases.

For \(\{x_0,x_R,x_C\}\), the first coordinate is the sign value \(\mathsf P\) or \(\mathsf N\).  Let \(Q\) be a hypothetical \(3\)-pseudo-cube on these three states, and take a first-coordinate value \(\alpha\in\{\mathsf P,\mathsf N\}\) that appears in \(Q\).  By \Cref{lem:pseudocube-tests}, the \(\alpha\)-slice of \(Q\), viewed on the row-column coordinates \(x_R,x_C\), is a \(2\)-pseudo-cube.  But every point in this slice must also be able to flip the \(x_0\)-coordinate while keeping \((r,s)\) fixed.  Thus each row-column pair in the slice must be a mixed-sign cell, so the slice is contained in the forest in \Cref{tab:counterexample-mixed-sign}, impossible.

For \(\{x_{\mathsf P},x_R,x_C\}\), the first coordinate takes values \(\mathsf P_1,\mathsf P_2,\bot\).  Let \(Q\) be a hypothetical \(3\)-pseudo-cube.  If the \(\bot\)-slice is nonempty, then it is a row-column \(2\)-pseudo-cube by \Cref{lem:pseudocube-tests}.  A \(\bot\)-value at \(x_{\mathsf P}\) comes from an \(\mathsf N\)-token, and flipping the \(x_{\mathsf P}\)-coordinate requires a \(\mathsf P\)-token in the same cell.  Hence the \(\bot\)-slice is contained in the mixed-sign forest in \Cref{tab:counterexample-mixed-sign}, impossible.

Therefore \(Q\) has no \(\bot\)-points.  Choose a surviving value \(v\in\{\mathsf P_1,\mathsf P_2\}\).  The \(v\)-slice is again a row-column \(2\)-pseudo-cube.  To see where this slice lives, take any point in it, say with row-column pair \((r,s)\).  This point has the form \((v,R_r,C_s)\).  Since \(Q\) is a \(3\)-pseudo-cube, this point must have an \(x_{\mathsf P}\)-coordinate flip while the \(x_R\)- and \(x_C\)-coordinates stay equal to \(R_r\) and \(C_s\).  The flipped value cannot be \(\bot\), because \(Q\) has no \(\bot\)-points.  Thus the same cell \(T_{r,s}\) must contain both \(\mathsf P_1\) and \(\mathsf P_2\).  Hence the \(v\)-slice is contained in the within-\(\mathsf P\) forest in \Cref{tab:counterexample-p-block}, impossible.

The case \(\{x_{\mathsf N},x_R,x_C\}\) is symmetric.  A nonempty \(\bot\)-slice would be a row-column \(2\)-pseudo-cube contained in the mixed-sign forest in \Cref{tab:counterexample-mixed-sign}: here the \(\bot\)-value at \(x_{\mathsf N}\) comes from a \(\mathsf P\)-token, and the \(x_{\mathsf N}\)-coordinate flip requires an \(\mathsf N\)-token in the same cell.  Thus no \(\bot\)-point remains.  For any surviving value \(v\in\{\mathsf N_1,\mathsf N_2\}\), each point \((v,R_r,C_s)\) in the \(v\)-slice must flip its \(x_{\mathsf N}\)-coordinate without changing \(R_r,C_s\).  Since \(\bot\)-points are absent, that flip must use the other \(\mathsf N\)-token in the same cell.  Hence the \(v\)-slice is contained in the within-\(\mathsf N\) forest in \Cref{tab:counterexample-n-block}, again impossible.

We have ruled out every possible three-state base witness.  Therefore \(\DSdim(\cH)\le2\).  Together with the lower bound, \(\DSdim(\cH)=2\), while \(\DSdim(\Roll_\halt(\cH))\ge3\).
\end{proof}

\end{document}